\documentclass[journal]{IEEEtran}
\usepackage{times}

\usepackage[utf8]{inputenc} % allow utf-8 input
\usepackage[T1]{fontenc}    % use 8-bit T1 fonts
\usepackage{hyperref}       % hyperlinks
\usepackage{url}            % simple URL typesetting
\usepackage{booktabs}       % professional-quality tables
\usepackage{amsfonts}       % blackboard math symbols
\usepackage{nicefrac}       % compact symbols for 1/2, etc.
\usepackage{microtype}      % microtypography

\usepackage{subfigure}
\usepackage{graphicx}
\usepackage{amssymb,amsmath}
\usepackage{xcolor}
\usepackage{algorithm,algorithmic}
\usepackage{multirow}
\usepackage[overload]{empheq}

\usepackage{epsfig,lscape}
\usepackage{caption}
\usepackage[resetlabels]{multibib}
\newcites{append}{Supplement References}

\usepackage{amsthm}
\newtheorem{prop}{Proposition}
\newtheorem{theorem}{Theorem}

\begin{document}
\title{Generative Modeling by Inclusive Neural Random Fields with Applications in Image Generation and Anomaly Detection}

% The \author macro works with any number of authors. There are two
% commands used to separate the names and addresses of multiple
% authors: \And and \AND.
%
% Using \And between authors leaves it to LaTeX to determine where to
% break the lines. Using \AND forces a line break at that point. So,
% if LaTeX puts 3 of 4 authors names on the first line, and the last
% on the second line, try using \AND instead of \And before the third
% author name.

\author{
	Yunfu~Song and
	Zhijian~Ou,~\IEEEmembership{Senior Member,~IEEE}
\IEEEcompsocitemizethanks{\IEEEcompsocthanksitem Yunfu Song and Zhijian Ou are with the Department of Electronic Engineering, Tsinghua university, Beijing, China.
	% note need leading \protect in front of \\ to get a newline within \thanks as
	% \\ is fragile and will error, could use \hfil\break instead.
	Email: ozj@tsinghua.edu.cn}% <-this % stops an unwanted space
\thanks{Manuscript received XX, 2019. Corresponding author: Zhijian Ou.}
}

%\iclrfinalcopy % Uncomment for camera-ready version, but NOT for submission.
%\part{title}
\maketitle
%\IEEEtitleabstractindextext{%
\begin{abstract}
Neural random fields (NRFs), referring to a class of generative models that use neural networks to implement potential functions in random fields (a.k.a. energy-based models), are not new but receive less attention with slow progress. Different from various directed graphical models such as generative adversarial networks (GANs), NRFs provide an interesting family of undirected graphical models for generative modeling.
In this paper we propose a new approach, the inclusive-NRF approach, to learning NRFs for continuous data (e.g. images), by introducing inclusive-divergence minimized auxiliary generators and developing stochastic gradient sampling in an augmented space.
Based on the new approach, specific inclusive-NRF models are developed and thoroughly evaluated in two important generative modeling applications - image generation and anomaly detection. 
The proposed models consistently improve over state-of-the-art results in both applications.
Remarkably, in addition to superior sample generation, one additional benefit of our inclusive-NRF approach is that, unlike GANs, it can directly provide (unnormalized) density estimate for sample evaluation.
With these contributions and results, this paper significantly advances the learning and applications of NRFs to a new level, both theoretically and empirically, which have never been obtained before.

\end{abstract}

%}

\section{Introduction}
	
\IEEEPARstart{G}{enerative} modeling is a way of statistical modeling that describes a joint probability distribution over target variables and has many applications in practice.
Recently, significant progress has been made on learning with deep generative models (DGMs), which usually refer to probabilistic models defined using multiple-layer neural networks (NNs).
There have emerged a bundle of deep directed generative models, such as variational AutoEncoders (VAEs) \cite{kingma2014auto-encoding}, generative adversarial networks (GANs) \cite{goodfellow2014generative,nowozin2016f-gan} and so on. 
In contrast, deep undirected generative models received less attention with slow progress in both learning algorithms and applications.

An undirected generative model, also known as random field (RF) \cite{koller2009probabilistic}, defines a probability distribution over variables $x$ in the form of $p_{\theta}(x)=\frac{1}{Z(\theta)} \exp\left[  u_{\theta}(x) \right]$, where $u_\theta(x)$ is  called the potential function with parameter $\theta$, and $Z(\theta)=\int\exp\left[  u_{\theta}(x) \right] dx$ is the normalizing constant.
In early days, the potential function $u_\theta(x)$ was often defined as linear functions, e.g. $u_\theta(x)=\theta^T f(x)$, where $f(x)$ is a vector of features (usually hand-crafted) and $\theta$ is the corresponding parameter vector. Such RFs are known as log-linear models \cite{koller2009probabilistic}.
Recently, there have been a number of studies on developing deep undirected generative models, which are characterized by employing multiple-layer neural networks to define the potential function $u_\theta(x)$.
These models appeared in different contexts with different specific model definitions, such as deep energy models (DEMs) \cite{ng11,Kim2016DeepDG}, descriptive models \cite{coopnets}, generative ConvNet \cite{wyn15}, neural random field language models \cite{asru}.
For ease of reference, we call such class of models as neural random fields (NRFs)\footnote{DEMs may refer to a broader class of models, allowing for non-probabilistic models \cite{lecun2006tutorial}.
	We use the termininology of NRFs to emphasize that we are mainly concerned with probabilistic models since divergence measures between probability distributions are used in our algorithmic development.
Also it should be emphasized that the neural random fields defined above, which basically are generative models, should not be confused with neural conditional random fields (CRFs) \cite{artieres2010neural,zheng2015conditional,hu2019neural}, which basically are discriminative models and  used for sequence labeling \cite{artieres2010neural,hu2019neural}, image segmentation \cite{zheng2015conditional} and so on.
} in general.

Conceptually, compared to log-linear RFs, if we could successful train such NRFs, we can jointly learn the features and the feature weights for generative modeling, which is highly desirable. However, learning NRFs presents much greater challenge, because the log-likelihood in NRFs is no longer concave (unlike in log-linear RFs) and the gradient involves the expectation with respect to (w.r.t.) the model distribution $p_\theta$.
Typically, one approximates the model expectation by Markov chain Monte Carlo (MCMC) sampling from $p_\theta$ to calculate stochastic gradients as in stochastic maximum likelihood (SML) \cite{younes1989parametric}.
Monte Carlo sampler thus is a crucial component which affects the learning of NRFs.

A recent progress in learning NRFs for generative modeling as studied in \cite{Kim2016DeepDG,coopnets,asru,Kuleshov2017NeuralVI} is to pair the target random field $p_\theta$ with an auxiliary directed generative model (often called generator) $q_\phi(x)$ parameterized by $\phi$, which approximates sampling from the target random field.
Learning is performed by maximizing the log-likelihood of training data under $p_\theta$ or some bound of the log-likelihood, and simultaneously minimizing some divergence between the target random field $p_\theta$ and the auxiliary generator $q_\phi$.
Different learning methods mainly differ in the objective functions used in the joint training of $p_\theta$ and $q_\phi$, and thus have different computational and statistical properties (partly illustrated in Figure \ref{fig:toy}).
For example, minimizing the exclusive-divergence $KL[q_\phi||p_\theta] \triangleq \int q_\phi \log \left( q_\phi / p_\theta \right) = - H \left[ q_\phi \right] - \int q_\phi \log p_\theta$ w.r.t. $\phi$, as employed in \cite{Kim2016DeepDG}, involves the intractable entropy term $H \left[ q_\phi \right]$ and tends to enforce the generator to seek modes, yielding missing modes.
Notably, there are also other factors that distinguish different studies in learning NRFs, e.g. modeling discrete or continuous data, different model choices of the target NRF and the auxiliary generator.

In this paper, we aim to advance the generative learning and applications of neural random fields for continuous data (e.g. images), 
by introducing inclusive-divergence minimized auxiliary generators $q_\phi$ and developing stochastic gradient sampling in an augmented space defined by both the target NRF and the auxiliary generator.
Minimizing the inclusive-divergence $KL[p_\theta||q_\phi] \triangleq \int p_\theta \log \left( p_\theta / q_\phi \right)$ w.r.t. $\phi$ along with stochastic gradient sampling in the augmented space offers several benefits.
First, inclusive minimization can avoid the annoying entropy term, which is suffered by minimizing the exclusive-divergence.
Second, inclusive minimization also tends to drive the auxiliary generator (acting like an adaptive proposal in adaptive MCMC \cite{andrieu2008tutorial,roberts2009examples}) to cover modes of the target density $p_\theta$, which is a desirable property for proposal design in MCMC.
Third, sampling in an augmented space could be more efficient, inspired from auxiliary variable MCMC \cite{neal2011mcmc}.
Therefore, our stochastic gradient sampler defined by both the target NRF and the auxiliary generator embodies both auxiliary variable MCMC and adaptive MCMC, which conceptually can realize more efficient sampling and thus help learning of NRFs.

To demonstrate the capability of NRFs as powerful generative models in applications, once being properly trained, two applications are examined. 
First, sample generation, e.g. image generation, is explored and evaluated as a basic application of various generative models for continuous data.
Second, note that a fundamental benefit of generative modeling by NRFs is that, unlike GANs and VAEs, NRFs can directly provides (unnormalized) density estimate.
Regarding this benefit, an interesting application of NRFs is anomaly detection, which basically can be addressed by density estimate --- anomalies are those ones residing in low probability density areas.
The unnormalized density estimates (as measured by potential values) provide a natural decision criterion for anomaly detection, since the normalizing constant only introduces a constant in thresholding.

The contributions of this work are two-fold.
First, we successfully develop the inclusive-NRF approach, which consists of proper designs of the target NRF model, the auxiliary generator and the sampler, as a whole, to learn NRFs for continuous data, in the framework of minimizing the inclusive-divergence between the target NRF and the auxiliary generator.
The differentiation and connection between our approach and prior work are summarized as follows, which are detailed in Section \ref{sec:related-work}.
\begin{itemize}
\item To our knowledge, the inclusive-NRF approach is the first in learning NRFs by minimizing the inclusive-divergence between the target NRF and the auxiliary generator for continuous data.

\item Particularly for continuous data (e.g. images), we develop stochastic gradient samplers in an augmented space, including but not limited to SGLD (stochastic gradient Langevin dynamics) \cite{sgld} and SGHMC (stochastic gradient Hamiltonian Monte Carlo) \cite{sghmc}, to exploit noisy gradients for NRF model sampling.
Notably, SGHMC improves over SGLD in learning NRFs.
Our SGLD/SGHMC sampler is not like in previous applications (\cite{sgld,sghmc}) which mainly simulate Bayesian posterior samples in large-scale Bayesian inference, though we use the same terminology.

\item Minor points: The target NRFs used in this work are different from those in \cite{Kim2016DeepDG,asru,coopnets}, though we use  latent-variable auxiliary generators, similar to \cite{coopnets}.
\end{itemize}

Second, we demonstrate the superior capability of the inclusive-NRF approach in applications of image generation and anomaly detection.
Comparisons between different DGMs are made with comparable network architectures.
It is found that the inclusive-NRFs consistently improve over state-of-the-art results in both applications, which are summarized as follows.
We will release our code and scripts for reproducing the results in this work.
\begin{itemize}
	\item 
	Inclusive-NRFs achieve state-of-the-art sample generation quality, measured by both Inception Score (IS) and Frechet Inception Distance (FID), obtaining IS 8.28 (FID 20.9) with Resnet architecture \cite{resnet} and {IS 7.54 (FID 27.9)} with CNN architecture \cite{cnn} on CIFAR-10.
	\item 
	By directly using the potential function for sample evaluation, inclusive-NRFs achieve state-of-the-art performance in anomaly detection on the widely benchmarked datasets - KDDCUP, MNIST, and CIFAR-10. 
	This shows that, unlike GANs, the new approach can provide informative density estimate, besides superior sample generation.
\end{itemize}

The remainder of this paper is organized as follows. 
After presenting background on random fields in Section \ref{sec:background}, we introduce the inclusive-NRF approach in Section \ref{sec:inclusive-NRF}.
In Section \ref{sec:related-work}, we discuss related work.
The extensive experimental evaluations are given in Sections  \ref{sec:experiments} .
We conclude the paper with a discussion in Section \ref{sec:discussion}.

\begin{figure*}[htb]
	\center
	\includegraphics[width=0.6\textwidth]{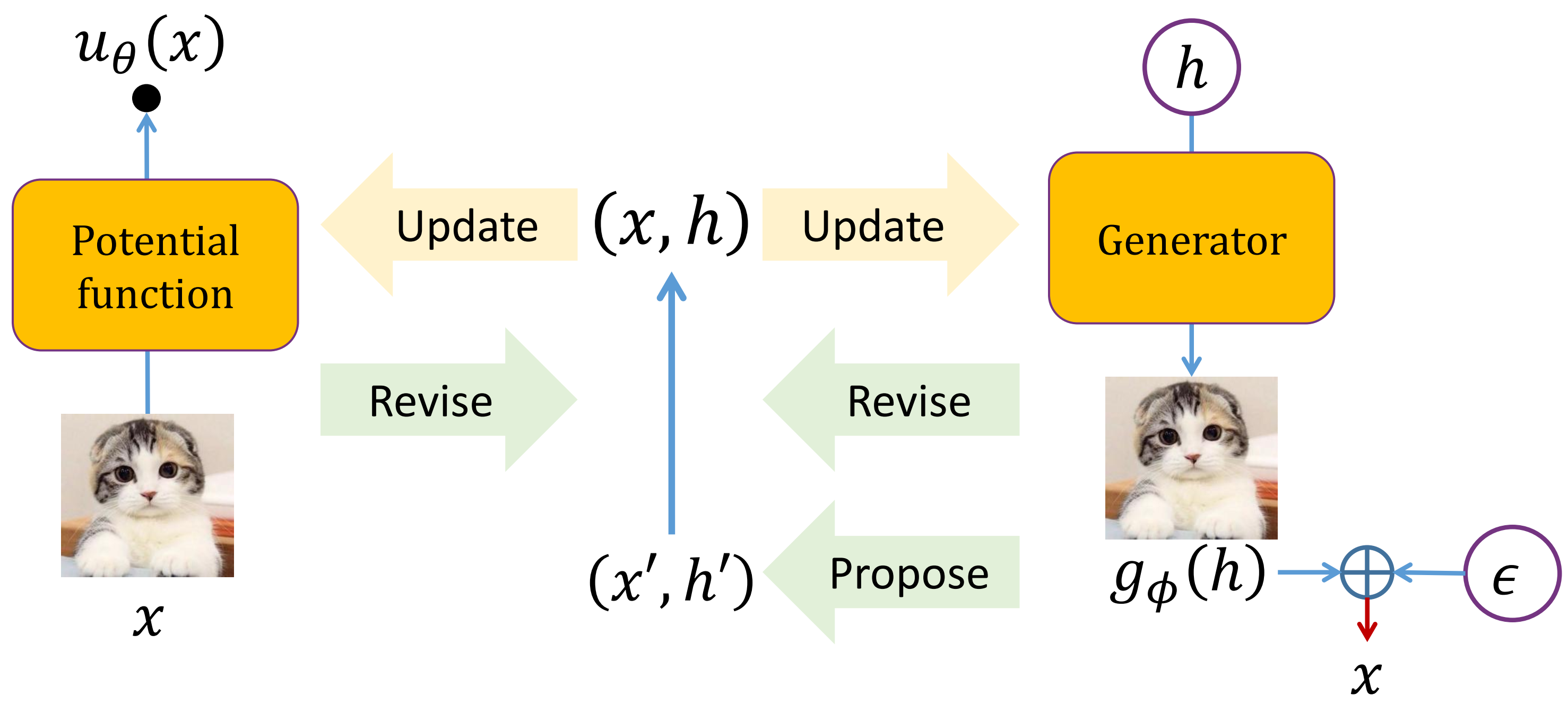}
	\caption{Overview of the inclusive-NRF approach. Two neural networks are used to define the NRF's potential function $u_{\theta}(x)$ and the auxiliary generator $g_\phi(h)$ respectively. The parameters of both networks, $\theta$ and $\phi$, are updated by using the revised samples $(x,h)$ in the augmented space, which are obtained by revising the samples $(x',h')$ proposed by the auxiliary generator, according to the stochastic gradients defined by both the target NRF and the auxiliary generator.} 
	\label{fig:inclusive-NRF}
\end{figure*}

\section{Background on Random fields} \label{sec:background}

Undirected models, or exchangeably termed as random fields form one of the two main classes of probabilistic graphical models \cite{frey2005comparison,koller2009probabilistic}.
In defining the joint distribution, directed models use conditional probability functions, with the directionality given by the conditioning relationship, whereas undirected models use unnormalized potential functions and are more suitable for capturing interactions among variables, especially when the directionality of a relationship cannot be clearly defined (e.g. as in between neighboring image pixels).

A random field (RF) defines a probability distribution for a collection of random variables $x\in \mathbb{R}^{d_x}$ with parameter $\theta$ in the form:
\begin{equation}\label{eq:unsup-RF}
p_{\theta}(x)=\frac{1}{Z(\theta)} \exp\left[  u_{\theta}(x) \right] 
\end{equation}
where $Z(\theta)=\int\exp\left[  u_{\theta}(x) \right] dx$ is the normalizing constant, $u_{\theta}(x)$ is called the potential function\footnote{Negating the potential function defines the energy function.} which assigns a scalar value to each configuration of $x$.
High probability configurations correspond to high potential/low energy configurations.

%There is an extensive literature devoted to maximum likelihood (ML) learning of random fields, as  reviewed briefly in \cite{Wang2017LearningTR} and more broadly in \cite{lecun2006tutorial}.

There is a large body of literatures devoted to learning random fields for generative modeling, among which the most acknowledged method is maximum likelihood (ML) learning and its variants.
The primary difficulty is that the gradient in maximizing the data log-likelihood $\log p_\theta(\tilde{x})$ for observed $\tilde{x}$ involves expectation w.r.t. the model distribution, as shown below:
\begin{equation} \label{eq:RF-grad}
\begin{aligned}
\nabla_\theta \log{p}_{\theta}(\tilde{x})&=\nabla_\theta u_{\theta}(\tilde{x})-\nabla_\theta \log Z(\theta)\\
&=\nabla_\theta u_{\theta}(\tilde{x})-E_{p_\theta(x)}\left[\nabla_\theta u_{\theta}(x)\right].
\end{aligned}
\end{equation}
%where we use the notation $\nabla_\theta$ for the partial derivative w.r.t. $\theta$.
Prior efforts to address this difficulty are further discussed and compared to our approach in Section \ref{sec:related-work}.

\section{The inclusive-NRF approach}
\label{sec:inclusive-NRF}

A high-level overview of our inclusive-NRF approach is shown in Figure \ref{fig:inclusive-NRF}. 
In the following, after introducing the NRF model (Section \ref{sec:NRF-model}), the two new designs - introducing the inclusive-divergence minimized auxiliary generator and developing stochastic gradient sampling are elaborated in Section \ref{sec:introduce-aux-generator} and Section \ref{sec:apply-SGLD} respectively.

\subsection{The NRF model} \label{sec:NRF-model}
To define a NRF over variables $x$, we implement the potential $u_{\theta}(x) : \mathbb{R}^{d_x} \rightarrow \mathbb{R}$, by a neural network, which takes the multi-dimensional $x \in \mathbb{R}^{d_x}$ as input and outputting the scalar $u_{\theta}(x) \in \mathbb{R}$. In this manner, we can take advantage of the representation power of neural networks for RF modeling. And such a RF essentially becomes defined over a fully-connected undirected graph and captures interactions in observations to the largest order, since the neural potential function $u_{\theta}(x)$ involves all the components in $x$.
Remarkably, the NRFs used in our experiments are different from similar models in previous studies \cite{Kim2016DeepDG,asru,coopnets}, as detailed in Section \ref{sec:related-work}.
	
\subsection{Introducing inclusive-divergence minimized auxiliary generators} \label{sec:introduce-aux-generator}

As shown in Eq. (\ref{eq:RF-grad}), the bottleneck in learning NRFs is that Monte Carlo sampling from the RF model $p_\theta$ is needed to approximate the model  expectation for calculating the gradient. 
Partly inspired from recent work in \cite{Kim2016DeepDG,coopnets,asru,Kuleshov2017NeuralVI} and, to be discussed in the end of this section, partly inspired from auxiliary variable MCMC \cite{neal2011mcmc} and adaptive MCMC \cite{andrieu2008tutorial,roberts2009examples}, we introduce an auxiliary generator to approximate sampling from the target RF.

In this paper, we are mainly interested in modeling fixed-dimensional continuous observations $x \in \mathbb{R}^{d_x}$ (e.g. images).
We use a directed generative model, $q_\phi(x,h) \triangleq q(h)q_\phi(x|h)$, for the auxiliary generator, which is defined as follows\footnote{Note that during training, $\sigma^2$ is absorbed into the learning rates and does not need to be estimated.}:
\begin{equation}\label{eq:generator}
\begin{aligned}
h &\sim \mathcal{N}(0,I_h),\\
x &= g_\phi(h)+\epsilon, \epsilon \sim \mathcal{N}(0,\sigma^2 I_{\epsilon}).
\end{aligned}
\end{equation}
Here $g_\phi(h):\mathbb{R}^{d_h} \rightarrow \mathbb{R}^{d_x}$ is implemented as a neural network with parameter $\phi$, which maps the latent code $h$ to the observation space.
$I_h$ and $I_{\epsilon}$ denote the identity matrices, with dimensionality implied by $h$ and $\epsilon$ respectively.
Drawing samples from the generator $q_\phi(x,h)$ is simple as it is just ancestral sampling from a 2-variable directed graphical model.

For dataset $\mathcal{D} = \left\lbrace \tilde{x}_1, \cdots, \tilde{x}_n \right\rbrace $, consisting of $n$ observations, let $\tilde{p}(\tilde{x}) \triangleq \frac{1}{n} \sum_{k=1}^{n} \delta(\tilde{x} - \tilde{x}_k)$ denotes the empirical data distribution.
A new design in this paper is that we perform the maximum likelihood learning of $p_\theta$ and simultaneously minimize the inclusive divergence between the target random field $p_\theta$ and the auxiliary generator $q_\phi$ by\footnote{Such optimization using two objectives is employed in a number of familiar learning methods, such as GAN with log$D$ trick \cite{goodfellow2014generative}, wake-sleep algorithm \cite{Hinton1995the}.}
\begin{equation}
\label{eq:jrf_unsup_obj}
\left\{
\begin{split}
& \min_{\theta} KL\left[  \tilde{p}(\tilde{x}) || p_\theta(\tilde{x}) \right] \\
& \min_{\phi} KL\left[  p_\theta(x) || q_\phi(x) \right] \\
\end{split}
\right.
\end{equation}
The first line of Eq. (\ref{eq:jrf_unsup_obj}) is equivalent to maximum likelihood training of the target RF $p_\theta$ under the empirical data $\tilde{p}$, which requires sampling from $p_\theta$.
Simultaneously, the second line optimizes the generator $q_\phi$ to be close to $p_\theta$ so that $q_\phi$ becomes a good proposal for sampling from $p_\theta$.
By the following Proposition \ref{prop:nrf_gradient_proof}, we can obtain the gradients w.r.t. $\theta$ and $\phi$ (to be ascended).
In practice, we apply minibatch based stochastic gradient descent (SGD) to solve the optimization problem Eq. (\ref{eq:jrf_unsup_obj}), as shown in Algorithm \ref{alg:learning-NRF-IAG}.

\begin{prop} \label{prop:nrf_gradient_proof}
The gradients for optimizing the two objectives in Eq. (\ref{eq:jrf_unsup_obj}) can be derived as follows:
\begin{equation}
\label{eq:jrf_unsup_gradient}
\left\{
\begin{split}
-\frac{\partial}{\partial \theta} KL&\left[  \tilde{p}(\tilde{x}) || p_\theta(\tilde{x}) \right] \\
&=E_{\tilde{p}(\tilde{x})}\left[\nabla_\theta u_\theta(\tilde{x})\right]-E_{p_\theta(x)}\left[\nabla_\theta u_\theta(x)\right]\\
-\frac{\partial}{\partial \phi} KL&\left[  p_\theta(x) || q_\phi(x) \right]\\
&=E_{p_\theta(x) q_\phi(h|x)}\left[ \nabla_\phi logq_\phi(x,h)\right]
\end{split}
\right.
\end{equation}
\end{prop}
\begin{proof}
	See Appendix \ref{sec:proof-prop-1}.
\end{proof}

Ideally, the learning of $\theta$ could be conducted without $\phi$, by using an MCMC sampler (e.g. LD) to draw samples from $p_\theta(x)$.
But the chain often mixes between modes so inefficiently that severely slow down the learning of $\theta$ especially when the target density $p_\theta(x)$ is multimodal. That is the main difficulty that hinders the effective training of NRFs.
In introducing auxiliary generator $q_\phi$ to approximate the target RF $p_\theta$, we are inspired by two advanced MCMC ideas - auxiliary variable MCMC \cite{neal2011mcmc} and adaptive MCMC \cite{andrieu2008tutorial,roberts2009examples}.
The classic example of adaptive MCMC is adaptive scaling of the variance of the step-size in random-walk Metropolis \cite{roberts2009examples}.
In our case, the auxiliary generator acts like an adaptive proposal, updated by using samples from the target density\footnote{Minimizing the inclusive-divergence tends to drive the generator (the proposal) to have higher entropy than the target density, which is a desirable property for proposal design in MCMC.}.
Further, to be detailed in Section \ref{sec:apply-SGLD}, the target density is extended to be $p_\theta(x) q_\phi(h|x)$, which leaves the original target as the marginal, but sampling in the augmented space $(x,h)$ can be easier (more efficiently), with the help of the adaptive proposal $q_\phi(x,h)$.
This follows the basic idea of auxiliary variable MCMC \cite{neal2011mcmc} - sampling in an augmented space could be more efficient.

\begin{algorithm*}[tb]
	\caption{Learning NRFs with inclusive auxiliary generators}
	\label{alg:learning-NRF-IAG}
	\begin{algorithmic}
		\REPEAT
		\STATE \underline{Sampling:}
		Draw a minibatch $\mathcal{M}=\left\lbrace (\tilde{x}^i,x^i,h^i), i=1,\cdots\,|\mathcal{M}|\right\rbrace $ from $\tilde{p}(\tilde{x}) p_\theta(x) q_\phi(h|x)$ (see Algorithm \ref{alg:model-sampling});
		
		\STATE \underline{Updating:}
		
		Update $\theta$ by ascending:		
		$\frac{1}{|\mathcal{M}|} \sum_{(\tilde{x},x,h) \sim \mathcal{M}}
		\left[\nabla_\theta u_\theta(\tilde{x}) - \nabla_\theta u_\theta(x) \right] $;
		
		Update $\phi$ by ascending:
		$
		\frac{1}{|\mathcal{M}|} \sum_{(\tilde{x},x,h) \sim \mathcal{M}}
		\nabla_\phi \log q_\phi(x,h) $;
		
		\UNTIL{convergence}
	\end{algorithmic}
\end{algorithm*}

\subsection{Developing stochastic gradient samplers for NRF model sampling} \label{sec:apply-SGLD}

In Algorithm \ref{alg:learning-NRF-IAG}, we need to draw samples $(x,h) \in \mathbb{R}^{d_x + d_h}$ in the augmented space defined by our target distribution $p_\theta(x) q_\phi(h|x)$ given current $\theta$ and $\phi$.
For such continuous distribution, samplers leveraging continuous dynamics (namely continuous-time Markov processes described by stochastic differential equations), such as Langevin dynamics (LD) and Hamiltonian Monte Carlo (HMC) \cite{neal2011mcmc}, are known to be efficient in exploring the continuous state space.
Simulating the continuous dynamics leads to the target distribution as the stationary distribution.
The Markov transition kernel defined by the continuous dynamical system usually involves using the gradients of the target distribution, which in our case are as follows:
\begin{equation} \label{eq:grad-x-h}
\left\{
\begin{split}
\frac{\partial}{\partial x} &\log \left[ p_\theta(x) q_\phi(h|x) \right] \\
&= \frac{\partial}{\partial x} \left[ \log  p_\theta(x) +  \log q_\phi(h,x) -  \log q_\phi(x) \right]\\
\frac{\partial}{\partial h} &\log \left[ p_\theta(x) q_\phi(h|x) \right] = \frac{\partial}{\partial h} \log q_\phi(h,x)
\end{split}
\right.
\end{equation}
It can be seen that it is straightforward to obtain the gradient w.r.t. $h$ and the first two terms\footnote{Notably, $\frac{\partial}{\partial x} \log p_\theta(x) = \frac{\partial}{\partial x} u_\theta(x)$ does not require the calculation of the normalizing constant.
} in the gradient w.r.t. $x$.
However, calculating the third term $\frac{\partial}{\partial x} \log q_\phi(x)$ in the gradient w.r.t. $x$ is intractable. Therefore we are interested in developing stochastic gradient variants of continuous-dynamics samplers, which rely on using noisy estimate of $\frac{\partial}{\partial x} \log q_\phi(x)$.

Recently, stochastic gradient samplers have emerged in simulating posterior samples in large-scale Bayesian inference, such as SGLD (stochastic gradient Langevin dynamics) \cite{sgld} and SGHMC (Stochastic Gradient Hamiltonian Monte Carlo) \cite{sghmc}.
To illustrate, consider the posterior $p(\theta|\mathcal{D})$ of model parameters $\theta$ given the observed dataset $\mathcal{D}$, with abuse of notation. 
We have $p(\theta|\mathcal{D}) \propto \exp \left[ \sum_{x \in \mathcal{D}} \log p_\theta(x) + \log p(\theta)  \right] $, which is taken as the target distribution.
Instead of using full-data gradients  $\frac{\partial}{\partial \theta} \log p(\theta|\mathcal{D})$, which needs a sweep over the entire dataset, these samplers subsample the dataset and use stochastic gradients
$
\frac{\partial}{\partial \theta} \left[ \frac{|\tilde{\mathcal{D}|}}{|\mathcal{D}|} \sum_{x \in \tilde{\mathcal{D}}} \log p_\theta(x) + \log p(\theta)  \right]
$ in the dynamic simulation, where $\tilde{\mathcal{D}} \subset \mathcal{D}$ is a subsampled data subset.
In this manner, the computation cost is significantly reduced in each iteration and such Bayesian inference methods scale to large datasets.

In practice, sampling is based on a discretization of the continuous dynamics. Despite the discretization error and the noise introduced by the stochastic gradients, it can be shown that simulating the discretized dynamics with stochastic gradients also leads to the target distribution as the stationary distribution, when the step sizes are annealed to zero at a certain rate. 
The convergence of SGLD/SGHMC can be obtained from \cite{Sato2014ApproximationAO,sghmc,ma2015complete}, as summarized in Theorem \ref{theorem:SGLD}.
	
	\begin{theorem} \label{theorem:SGLD}
		Denote the target density as $p(z;\lambda)$ with given $\lambda$.
		Assume that one can compute a noisy, unbiased estimate $\Delta(z;\lambda)$ (a stochastic gradient) to the gradient $\frac{\partial}{\partial z} \log p(z;\lambda)$.
		For a sequence of asymptotically vanishing time-steps $\left\lbrace \delta_l, l \ge 1 \right\rbrace$ (satisfying $\sum_{l=1}^\infty \delta_l = \infty$ and $\sum_{l=1}^\infty \delta_l^2 < \infty$) and an i.i.d. noise sequence $\eta^{(l)}$, the SGLD iterates as follows, starting from $z^{(0)}$:
		\begin{equation}\label{eq:SGLD}
		\begin{aligned}
		z^{(l)} =& z^{(l-1)}
		+ \delta_l \Delta (z^{(l-1)};\lambda)
		+ \sqrt{2\delta_l} \eta^{(l)},\\
		&\quad\eta^{(l)} \sim \mathcal{N}(0,I), l=1,\cdots\,
		\end{aligned}
		\end{equation}
		
		Starting from $z^{(0)}$ and $v^{(0)}=0$, the SGHMC iterates as follows:
		\begin{equation}
		\label{eq:SGHMC}
		\left\{
		\begin{split}
		v^{(l)} =& (1-\beta) v^{(l-1)}
		+ \delta_l \Delta (z^{(l-1)};\lambda)
		+ \sqrt{2\beta\delta_l} \eta^{(l)},\\
		&\quad\eta^{(l)} \sim \mathcal{N}(0,I)\\
		z^{(l)} =& z^{(l-1)} + v^{(l)}, l=1,\cdots\,
		\end{split}
		\right.
		\end{equation}

The iterations of Eq. (\ref{eq:SGLD}) and (\ref{eq:SGHMC}) lead to the target distribution $p(z;\lambda)$ as the stationary distribution.
	\end{theorem}

By considering $z \triangleq (x,h)$, $p(z; \lambda) \triangleq p_\theta(x) q_\phi(h|x)$, $\lambda \triangleq (\theta, \phi)^T$, and Eq. (\ref{eq:grad-x-h}),
we can use Theorem \ref{theorem:SGLD} to develop the sampling step for Algorithm \ref{alg:learning-NRF-IAG}, as presented in Algorithm \ref{alg:model-sampling}.
For the gradient w.r.t. $x$, the intractable term $\frac{\partial}{\partial x} \log q_\phi(x)$ is estimated by a stochastic gradient.

\begin{prop} \label{prop:unbiased-est}
	Given $q_\phi(h,x)$, we have
	\begin{equation} \label{eq:unbiased-est}
	\frac{\partial}{\partial x} \log q_\phi(x)
	= E_{h^* \sim q_\phi(h^*|x)} \left[ \frac{\partial}{\partial x} \log q_\phi(h^*,x) \right].
	\end{equation}
\end{prop}
\begin{proof}
	See Appendix \ref{sec:proof-prop-2}.
\end{proof}

Motivated by Proposition \ref{prop:unbiased-est},
ideally we draw $h^* \sim q_\phi(h^*|x)$ and then use $\frac{\partial}{\partial x} \log q_\phi(h^*,x)$ as an unbiased estimator of $\frac{\partial}{\partial x} \log q_\phi(x)$.
In practice, at step $l$, given $x^{(l-1)}$ and starting from $h^{(l-1)}$, we run one step of LD sampling over $h$ targeting $q_\phi(h|x^{(l-1)})$, to obtain $h^{(l-1)*}$ and calculate $\frac{\partial}{\partial x^{(l-1)}} \log q_\phi(h^{(l-1)*},x^{(l-1)})$. This gives a biased but tractable estimator to  $\frac{\partial}{\partial x} \log q_\phi(x)$. It is empirically found in our experiments that more steps of this inner LD sampling do not significantly improve the performance for NRF learning.

	\begin{algorithm*}[tb]
	\caption{Sampling in the augmented space defined by $p_\theta(x) q_\phi(h|x)$}
	\label{alg:model-sampling}
	\begin{algorithmic}
		\STATE 1. Conduct ancestral sampling from the auxiliary generator $q_\phi(x,h)$, i.e. first draw $h' \sim q(h')$, and then draw $x' \sim q_\phi(x'|h')$;
		\STATE 2. Starting from $(x',h') = z^{(0)}$, run finite steps of SGLD/SGHMC $(l=1,\cdots,L)$ to obtain $(x,h)=z^{(L)}$, which we call \emph{sample revision}, according to Eq. (\ref{eq:SGLD}) or (\ref{eq:SGHMC}). 
		\STATE In particular, the SGLD recursions are conducted as follows:
		%Starting from $h^{(l-1)}$, run finite steps (empirically ) of LD
		\begin{equation} \label{eq:SGLD-specific}
		\left\{
		\begin{split}
		x^{(l)} =& x^{(l-1)}
		+ \delta_l \frac{\partial}{\partial x^{(l-1)}} \left[ \log  p_\theta(x^{(l-1)}) +  \log q_\phi(h^{(l-1)},x^{(l-1)}) -  \log q_\phi(h^{(l-1)*},x^{(l-1)}) \right]
		+ \sqrt{2\delta_l} \eta_x^{(l)},\\
		h^{(l)} =& h^{(l-1)}
		+ \delta_l \frac{\partial}{\partial h^{(l-1)}} \log q_\phi(h^{(l-1)},x^{(l-1)})
		+ \sqrt{2\delta_l} \eta_h^{(l)}, \quad\eta^{(l)}\triangleq (\eta_x^{(l)}, \eta_h^{(l)})^T \sim \mathcal{N}(0,I)\\
		\end{split}
		\right.
		\end{equation}
		\STATE and the SGHMC recursions are conducted as follows:
		\begin{equation} \label{eq:SGHMC-specific}
		\left\{
		\begin{split}
		v_x^{(l)}=&(1-\beta)v_x^{(l-1)}+\delta_l \frac{\partial}{\partial x^{(l-1)}} \left[ \log  p_\theta(x^{(l-1)}) + \log q_\phi(h^{(l-1)},x^{(l-1)}) - \log q_\phi(h^{(l-1)*},x^{(l-1)}) \right]+ \sqrt{2\beta\delta_l} \eta_x^{(l)},\\
		v_h^{(l)}=&(1-\beta)v_h^{(l-1)}+\delta_l \frac{\partial}{\partial h^{(l-1)}} \log q_\phi(h^{(l-1)},x^{(l-1)})
		+ \sqrt{2\beta\delta_l} \eta_h^{(l)},\\
		x^{(l)} =& x^{(l-1)}+v_x^{(l)}, \quad h^{(l)} = h^{(l-1)}+ v_h^{(l)}, \quad\eta^{(l)}\triangleq (\eta_x^{(l)}, \eta_h^{(l)})^T \sim \mathcal{N}(0,I)\\
		\end{split}
		\right.
		\end{equation}
		\STATE 
		 where, for $l > 1$, $h^{(l-1)*}$, which is an approximate sample from $q_\phi(h|x^{(l-1)})$ given $x^{(l-1)}$, is obtained from running one step of LD as follows, starting from $h^{(l-1)}$:
		\begin{equation} \label{eq:h-LD}
		h^{(l-1)*}=h^{(l-1)}+ \delta_{l}^{*} \frac{\partial}{\partial h^{(l-1)}} \log q_\phi(h^{(l-1)},x^{(l-1)})
		+ \sqrt{2\delta_{l}^{*}} \eta_{h}^{(l)*}, \quad \eta_h^{(l)*} \sim \mathcal{N}(0,I);
		\end{equation}
		for $l=1$, we directly use $h^{(0)}$ as $h^{(0)*}$, since, by initialization, $h^{(0)}$ is an exact sample from $q_\phi(h|x^{(0)})$ given $x^{(0)}$.
		\STATE \textbf{Return} $(x,h)$, i.e. $z^{(L)}$.
	\end{algorithmic}
\end{algorithm*}

So instead of using the exact gradient $\frac{\partial}{\partial z} \log p(z;\lambda)$ as shown in Eq. (\ref{eq:grad-x-h}) in our case, we develop a tractable biased stochastic gradient $\Delta(z;\lambda)$ as follows:
\begin{equation}  \label{eq:stochastic-grad}
\Delta(z;\lambda) \triangleq
\left( \begin{array}{c}
\frac{\partial}{\partial x} \left[ \log  p_\theta(x) +  \log q_\phi(h,x) -  \log q_\phi(h^*,x) \right] \\
\frac{\partial}{\partial h} \log q_\phi(h,x)
\end{array} \right),
\end{equation}
where $h^*$ is an approximate sample from $q_\phi(h^*|x)$ obtained by running one step of LD from $(h,x)$.
Remarkably, as we show in Algorithm \ref{alg:model-sampling}, the starting point $(h^{(0)}, x^{(0)})$ for the SGLD/SGHMC recursions is obtained from an ancestral sampling from $ q_\phi(h,x)$. Thus at step $l=1$, $h^{(0)}$ is already a sample from $q_\phi(h|x^{(0)})$ given $x^{(0)}$, and we can directly use $h^{(0)}$ as $h^{(0)*}$ without running the inner LD sampling.
Afterwards, for $l > 1$, the conditional distribution of $h^{(l-1)}$ given $x^{(l-1)}$ is close to $q_\phi(h|x^{(l-1)})$, though strictly not.
We could run one or more steps of LD to obtain $h^{(l-1)*}$ to reduce the bias in the stochastic gradient estimator.

With the above stochastic gradients in Eq. (\ref{eq:stochastic-grad}), the sampling step in Algorithm \ref{alg:learning-NRF-IAG} can be performed by running $|\mathcal{M}|$ parallel chains, each chain being executed by running finite steps of SGLD/SGHMC with tractable gradients w.r.t. both $x$ and $h$, as shown in Algorithm \ref{alg:model-sampling}. 
Intuitively, the auxiliary generator first gives a proposal $(x',h')$, and then the system follows the gradients of $p_\theta(x)$ and $q_\phi(h,x)$ (w.r.t. $x$ and $h$ respectively) to revise $(x',h')$ to $(x,h)$.
The gradient terms pull samples moving to low energy region of the random field and adjust the latent code of the generator, while the noise term brings randomness.
In this manner, we obtain Markov chain samples in the augmented space defined by $p_\theta(x) q_\phi(h|x)$.
Finally note that, as discussed before, finite steps in Eq. (\ref{eq:SGLD-specific})-(\ref{eq:h-LD}) in applying SGLD/SGHMC sampling from $p_\theta(x) q_\phi(h|x)$ will produce biased estimates of the gradients in Eq. (\ref{eq:jrf_unsup_gradient}) for NRF learning.
We did not find this to pose problems to the SGD optimization in practice, as similarly found in \cite{cd,bornschein2014reweighted,Kuleshov2017NeuralVI}, which work with biased gradient estimators.

\begin{figure}[htb]
	\center
	\includegraphics[width=0.5\textwidth]{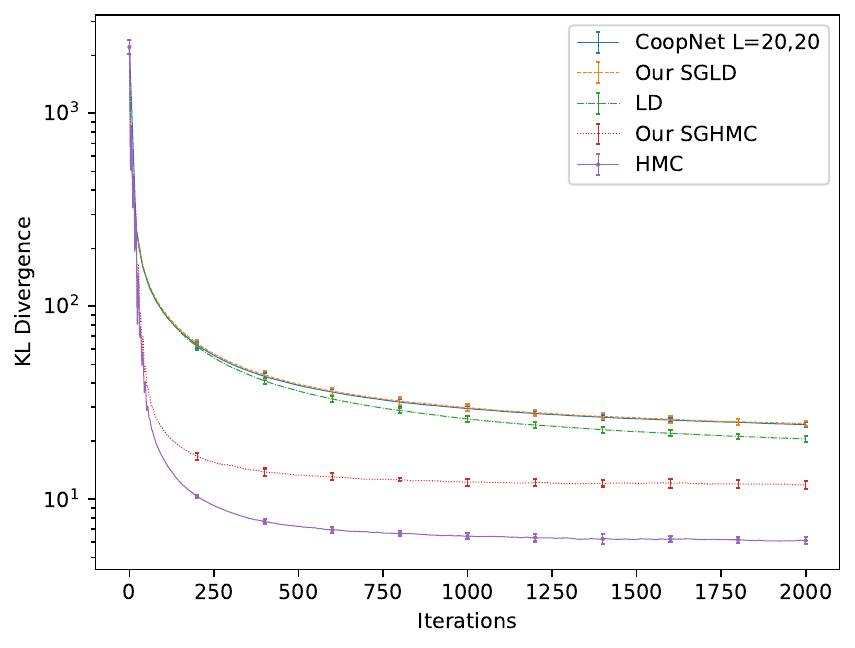}
	\caption{Sampler’s performance measured by the KL divergence with 10 independent runs to obtain standard deviations.
		``CoopNet $L=20,20$'' denotes the sampling method in \cite{coopnets} with $(L_x=20, L_h=20)$.
		``LD'' or ``HMC'' means the Langevin Dynamics or Hamiltonian Monte Carlo sampling of the target distribution $p_\theta(x)q_\phi(h|x)$ with exact gradients.
		``Our SGLD'' or ``Our SGHMC'' are our developed samplers with stochastic gradients (Algorithm \ref{alg:model-sampling}).
		We fix the total iterations of $x$ and $h$ to be the same for each sampling method. Thus one iteration of ``CoopNet $L=20,20$'' would be regarded as 20 iterations of other methods in the figure.}
	\label{fig:sampler}
\end{figure}

\subsubsection{Evaluation of the stochastic gradient samplers}

To examine the sampling performance of our  SGLD/SGHMC samplers, we conduct a synthetic experiment and the results are shown in Figure \ref{fig:sampler}.
The $p_\theta(x)$ and $q_\phi(x,h)$ are 50D and 100D  Gaussians respectively with randomly generated covariance matrices (i.e. both $x$ and $h$ are of 50D).
For evaluation, we simulate $K = 500$ parallel chains for $T = 2000$ steps.
We follow \cite{ma2015complete} to evaluate the sampler’s performance measured by the KL divergence from the empirical Gaussian (estimated by the samples) to the ground truth $p_\theta(x) q_\phi(h|x)$.
We use the stepsize schedule of $\delta _ { t } = \left( a \cdot \left( 1 + \frac { t } { b } \right) \right) ^ { - c }$ like in \cite{ma2015complete} with $(a=10, b=1000, c=2)$ for all methods, 
and $\beta=0.1$ for SGHMC, and we find that these hyperparameters perform well for each method during the experiment.
The main observations are as follows. 
First, SGLD and SGHMC converge, though worse than their counterparts using exact gradients (LD and HMC).
Second, HMC samplers, whether using exact gradients or using stochastic gradients, outperform the corresponding LD samplers, since HMC dynamics, also referred to as second-order Langevin dynamics, exploit an additional momentum term.
Third, interestingly, the SGHMC sampler outperforms the LD sampler with exact gradients.
This reveals the benefit of our systematic development of the stochastic gradient samplers, including but not limited to SGLD and SGHMC.
Although the CoopNet sampler in \cite{coopnets} (to be described in Section \ref{sec:related-work}) performs close to our SGLD sampler, its performance is much worse than our SGHMC sampler. Our SGHMC is a new development, which cannot be obtained from simply extending the CoopNet sampler.

\section{Related work}
\label{sec:related-work}

Comparison and connection of our inclusive-NRF approach with related work are provided in the following from four perspectives, which reveal our major contribution on top of prior work - the first in learning NRFs for continuous data by minimizing the inclusive-divergence and developing stochastic gradient sampling, with theoretical analysis and  state-of-the-art performance in applications of image generation and anomaly detection as shown in experiments.

\subsection{Learning NRFs}

Learning random fields (RFs) for generative modeling remains to be a challenging problem after many years of research with many important methods.
The primary difficulty is that calculating the gradient of log-likelihood requires the intractable expectation w.r.t. the RF model $p_\theta$.
We classify two classes of methods to address this difficulty, depending on whether additional auxiliary models are introduced or not, apart from the target NRFs.

\subsubsection{Learning without auxiliary models}

An important class of methods in learning RFs is based on stochastic approximation (SA) \cite{SA51}, which approximates the model expectation by Monte Carlo sampling, and iterates Monte Carlo sampling from current $p_\theta$ and parameter updating of $\theta$.
The classic SA-based algorithm, initially proposed in \cite{younes1989parametric}, is often called stochastic maximum likelihood (SML).
In the literature on training restricted Boltzmann machines (RBMs), SML is also known as persistent contrastive divergence (PCD) \cite{tieleman2008training} to emphasize that the Markov chain (MC) is persistent between parameter updates, which yields unbiased gradient estimates.
The contrastive divergence (CD) method \cite{cd} initializes the finite-step MC sampling from the observed samples, thus losing the chain persistence and yielding biased estimates of gradients in general \cite{carreira2005contrastive}. Empirically, CD has been shown to be effective in training RBMs.
Recently in \cite{IGG-EBM} (IGG-EBM), a replay buffer of past generated samples is maintained. Langevin dynamics is initialized from the replay buffer 95\% of the time and from uniform noise otherwise, which heuristically may reduce mixing times between chains. However, IGG-EBM \cite{IGG-EBM} only achieved moderate image generation performance over CIFAR-10, even using 60 steps of Langevin dynamics to generate model samples in training, which is computational much more expensive than our approach.

Alternative objectives other than maximum-likelihood exist, such as in pseudolikelihood \cite{besag1975statistical}, score matching (SM) \cite{hyvarinen2005estimation}, noise contrastive estimation (NCE) \cite{gutmann2012noise}, and minimum
probability flow \cite{sohl2009minimum}.
These methods are typically reported to work well for low-dimensional data with simple potentials \cite{song2019generative}.
There exist some recent progress along SM and NCE.

Basically, SM works by matching the model score function and the data score function\footnote{The score function of a density of $x$ is defined as the gradient of the log-density w.r.t. $x$, as defined in \cite{hyvarinen2005estimation}}.
Recently, in \cite{song2019generative}, a multiple-layer NN, called noise conditional score network (NCSN), is trained to approximate the model score function. After training, NCSN uses an annealed Langevin dynamics to produce descent images, using a total of $10 \times 1000$ Langevin steps with 10 noise levels. In contrast to our approach, NCSN is computational much more expensive in image generation and loses the capability in providing (unnormalized) density estimate (since it only trains the score function). Thus NCSN cannot be applied in anomaly detection, which requires density estimate.

Basically, NCE works by discriminating between data samples (evaluated by model distribution) and noise samples drawn from a known noise distribution\footnote{In some sense, NCE could also be classified into the second class of methods which use auxiliary models.}.
Recently, dynamic NCE (DNCE) \cite{wang2018improved} improves NCE by introducing a dynamic noise distribution and using the interpolation of the data distribution and the dynamic noise distribution to train the discriminator.
NRF language models are successfully trained by DNCE \cite{wang2018improved}. 
Notably, NCE requires both easy sampling and likelihood evaluation from the noise model, and thus may pose some limitation in applications. For example, NCE cannot use latent-variable auxiliary generators as noise models, because likelihood evaluation is intractable for such noise models.

\subsubsection{Learning with auxiliary models}

A recent progress in learning NRFs as studied in \cite{Kim2016DeepDG,coopnets,asru,Kuleshov2017NeuralVI} is to jointly train the target random field $p_\theta(x)$ and an auxiliary generator $q_\phi(x)$.
In practice, the performance of learning NRFs with auxiliary models generally performs better than without auxiliary models.
Different studies mainly differ in the objective functions used in the joint training, and thus have different computational and statistical properties.

It is shown in Proposition \ref{prop:Kim-Bengio} in Appendix \ref{sec:proof-kim-bengio} that learning in \cite{Kim2016DeepDG} minimizes the exclusive-divergence $KL[q_\phi||p_\theta]$ w.r.t. $\phi$, which involves the intractable entropy term $H \left[ q_\phi \right]$ and tends to enforce the generator to seek modes, yielding missing modes. We refer to this approach as exclusive-NRF.

Learning in \cite{asru} and in this paper minimizes the inclusive-divergence $KL[p_\theta||q_\phi]$ w.r.t. $\phi$.
	%	But noticeably, this paper is not an extension of \cite{asru} which is developed for discrete data, but rather represents an important parallel development for continuous data.
	But noticeably, this paper presents our innovation in development of NRFs for continuous data, which is fundamentally different from \cite{asru} for discrete data.	
	The target NRF model, the generator and the sampler  all require new designs.
	\cite{asru} mainly studies random field language models, using LSTM generators (autoregressive with no latent variables) and employing Metropolis independence sampler (MIS) - applicable for discrete data (natural sentences).
	In this paper, we design random field models for continuous data (e.g. images), choosing latent-variable generators and developing SGLD/SGHMC to exploit noisy gradients in the continuous space.

	In \cite{coopnets} (CoopNet), motivated by interweaving maximum likelihood training of the random field $p_\theta$ and the latent-variable generator $q_\phi$, a joint training method is introduced to train NRFs.
	There are clear differences that distinguish our inclusive-NRF approach from CoopNet in both the sampler and the objective function.
	First, CoopNet uses LD sampling to generate samples, but two LD sampling steps are intuitively interleaved according to $ \frac{\partial}{\partial x} \log  p_\theta(x)$ (with $L_x$ steps) and $\frac{\partial}{\partial h} \log q_\phi(h,x)$ (with $L_h$ steps) separately, not aiming to draw samples from $p_\theta(x) q_\phi(h|x)$.
This is different from our stochastic gradient sampler in the augmented space, which moves $(x,h)$ jointly, as systematically developed in Section \ref{sec:apply-SGLD}.
Moreover, our SGHMC is a new development, which cannot be obtained from simply extending the CoopNet sampler.
	Second, according to theoretical understanding in \cite{coopnets}, Coopnet targets the following joint optimization problem:
	\begin{displaymath}
	\label{eq:coopnet_obj}
	\left\{
	\begin{split}
	& \min_{\theta} \left\lbrace KL\left[  \tilde{p}(\tilde{x}) || p_\theta(\tilde{x}) \right] - KL\left[ r(h,x)  || p_\theta(x) \right]\right\rbrace \\
	& \min_{\phi} KL\left[  r(h,x) || q_\phi(h,x) \right] \\
	\end{split}
	\right.
	\end{displaymath}
where $r(h,x)$ denotes the distribution of $(x^{(L_x)}, h^{(L_h)})$, resulting from the CoopNet sampler.
	This objective is also clearly different from our learning objective as shown in Eq. (\ref{eq:jrf_unsup_obj}), which aims to minimize the inclusive-divergence $KL[p_\theta||q_\phi]$ w.r.t. $\phi$.
	Regarding performance comparison, firstly the CoopNet sampler performs much worse than our SGHMC sampler, as shown in Figure \ref{fig:sampler}.
	It is further shown in Table \ref{is_unsupervised} that inclusive-NRF with SGLD outperforms CoopNet in image generation, and in Table \ref{tab:ablation} that utilizing SGHMC in learning inclusive-NRFs to exploit gradient information with momentum yields further better performance than using SGLD.
	%	Let $K^p_\theta, K^q_\phi$ denote the Markov transition kernel of Langevin dynamics that samples $x \sim p_\theta(x)$ and $h \sim q_\phi(h|x)$ respectively. Let $K^q_\phi \circ K^p_\theta \circ q_\phi(h,x)$ be the distribution obtained by running the Markov transition starting from $(h,x) \sim q_\phi(h,x)$.	
	%	It is shown in \cite{coopnets} that the training method converges to a fixed point $(\hat{\theta}, \hat{\phi})$:
	%	\begin{displaymath}
	%	\begin{aligned}
	%	\hat{\theta} &= \arg \min_{\theta} \left\lbrace  KL\left[  \tilde{p} || p_\theta \right] - KL\left[  \mathcal{M}_{\hat{\theta}}q_{\hat{\phi}} || p_\theta \right] \right\rbrace ,\\
	%	\hat{\phi} &= \arg \min_{\phi} KL\left[  \mathcal{M}_{\hat{\theta}}q_{\hat{\phi}} || q_\phi \right].
	%	\end{aligned}
	%	\end{displaymath}

Learning in \cite{Kuleshov2017NeuralVI} minimizes the $\chi^2$-divergence $\chi^2[q_\phi||p_\theta] \triangleq \int \frac{(p_\theta-q_\phi)^2}{q_\phi} $ w.r.t. $\phi$, which also tends to drive the generator to cover modes. But this approach is severely limited by the high variance of the gradient estimator w.r.t. $\phi$, and is only tested on the simpler MNIST and Omniglot.

Learning in \cite{han2019divergence} further introduces an inference model, apart from the target NRF and the  latent-variable generator, and jointly optimizes the three models under a divergence triangle.
Such study of augmenting the NRF with the inference capability over latent variables is interesting but outside the scope of current work.

Additionally, different NRF studies also distinguish in models used in the joint training with auxiliary models. The target NRF used in this work is different from those in previous studies \cite{Kim2016DeepDG,asru,coopnets}. The differences are: \cite{Kim2016DeepDG} includes additional linear and squared terms in $u_{\theta}(x)$, \cite{asru} defines over discrete-valued sequences, and \cite{coopnets} defines in the form of exponential tilting of a reference distribution (Gaussian white noise).
There also exist different choices for the generator, such as GAN models in \cite{Kim2016DeepDG}, LSTMs in \cite{asru}, or latent-variable models in  \cite{coopnets} and this work.

\subsection{Monte Carlo sampling}

Monte Carlo sampler is a crucial component in learning NRFs to approximate the model expectation.
For continuous data, gradient-based MCMC method - LD has been used in \cite{xie2016theory,IGG-EBM} and HMC in \cite{teh2003energy,ng11} to sample from the RF model $p_\theta$.
Like most MCMC methods, LD exhibits high auto-correlation and has difficulty mixing between separate modes. HMC improves over LD but still struggle to move well between modes, especially when the dimensionality is high \cite{neal2011mcmc}.

We draw inspiration from auxiliary variable MCMC  \cite{neal2011mcmc} to sample in an augmented space.
HMC \cite{neal2011mcmc} is a classic auxiliary variable MCMC which introduces auxiliary momentum variable.
Our approach introduces and simultaneously learns an auxiliary generator $q_\phi(x,h)$ to capture latent low-dimensional structure in the target NRF, and develop stochastic gradient sampler for the joint variable $(x,h)$.
We also draw inspiration from adaptive MCMC \cite{andrieu2008tutorial,roberts2009examples} to use adaptive proposal. The auxiliary generator acts like an adaptive proposal.
Therefore, our sampler embodies both auxiliary variable MCMC and adaptive MCMC. This feature presumably explains why our approach achieves  efficient sampling and thus successful learning of NRFs.

Recently, there are several efforts to develop advanced MCMC samplers that, like our sampler, embody both auxiliary variable MCMC and adaptive MCMC, such as in \cite{Levy2017GeneralizingHM,habib2018auxiliary}.
The L2HMC \cite{Levy2017GeneralizingHM} learns a parametric leapfrog operator to extend HMC.
The auxiliary variational MCMC \cite{habib2018auxiliary} is similar to our sampler, but optimizes the auxiliary generator by mimimizing exclusive-divergence.
To improve their sampler, a potential avenue for future work said in \cite{habib2018auxiliary} is to use inclusive-divergence minimization to ensure good coverge of the auxiliary generator. That is almost what we do in this work.
Notably, the above samplers \cite{Levy2017GeneralizingHM,habib2018auxiliary} are mainly evaluated for sampling from a given model, and not for learning NRFs.

\subsection{Comparison and connection with GANs}
GANs are state-of-the-art models for image generation, but lack in providing density estimate.
On the one hand, there are some efforts that aim to address the inability of GANs to provide sensible energy estimates for samples. The energy-based GANs (EBGAN) \cite{ebgan} proposes to view the discriminator as an energy function by designing an auto-encoder discriminator.
The recent work in \cite{Dai2017CalibratingEG} connects \cite{ebgan} and \cite{Kim2016DeepDG}, and show another two approximations for the entropy term.
However, it is known that as the generator converges to the true data distribution, the GAN discriminator converges to a degenerate uniform solution. This basically afflicts the GAN discriminator to provide density information, though there are some modifications.
In contrast, our inclusive-NRFs, unlike GANs, naturally provide (unnormalized) density estimate, which is examined with GMM synthetic experiments and anomaly detection benchmarking experiments.

	\begin{figure*}
	\vskip -0.2in
	\subfigure[Training data]{
		\begin{minipage}{0.23\textwidth}
			\flushleft  
			\includegraphics[width=\textwidth]{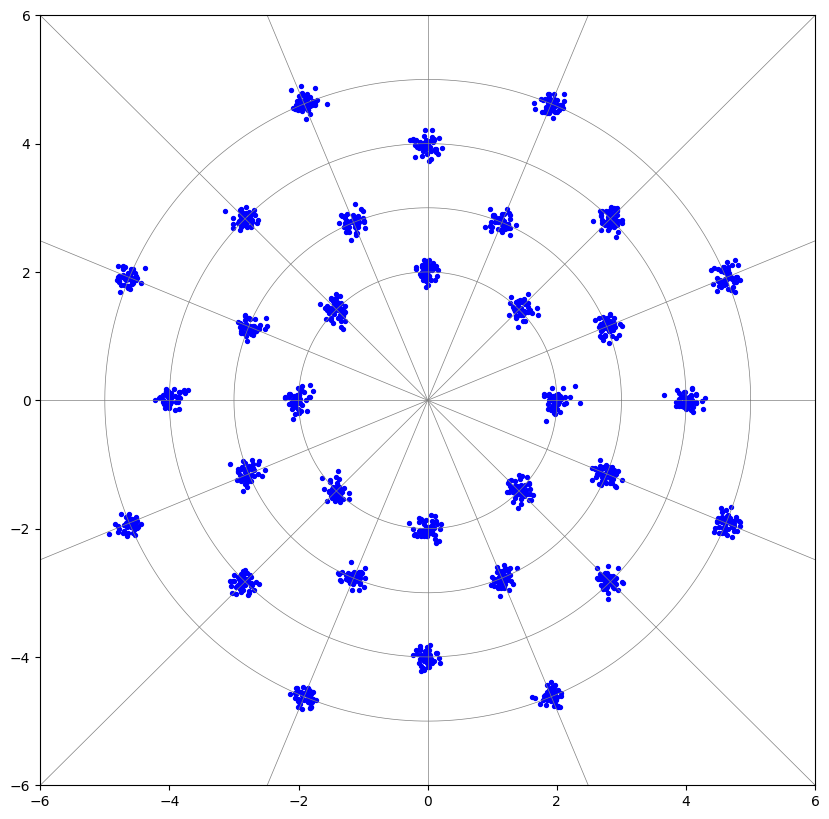}  
	\end{minipage}}
	\subfigure[GAN generation]{
		\begin{minipage}{0.23\textwidth}  
			\flushleft
			\includegraphics[width=\textwidth]{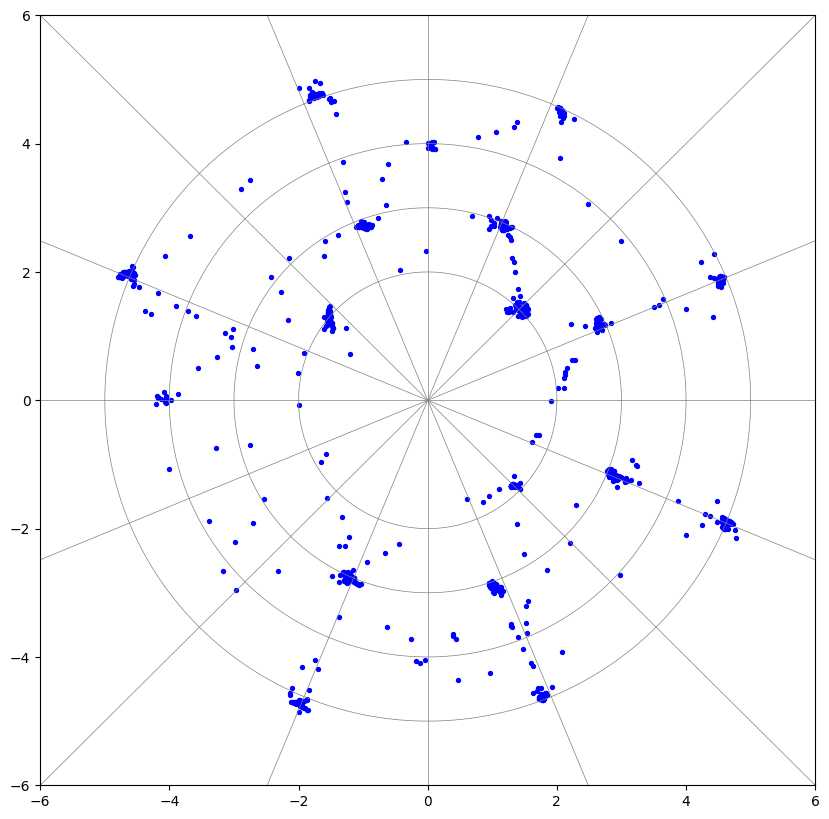}  
	\end{minipage}}
	\subfigure[WGAN-GP generation]{  
		\begin{minipage}{0.23\textwidth}  
			\flushleft  
			\includegraphics[width=\textwidth]{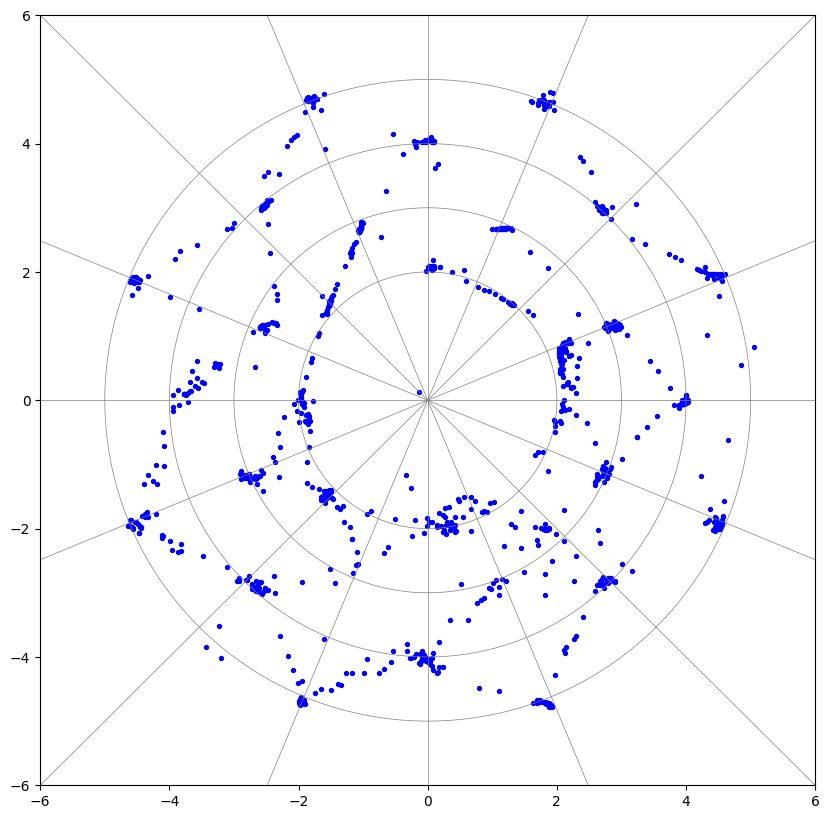}    
	\end{minipage}}
	\subfigure[Exclusive-NRF generation]{  
		\begin{minipage}{0.23\textwidth}  
			\flushleft  
			\includegraphics[width=\textwidth]{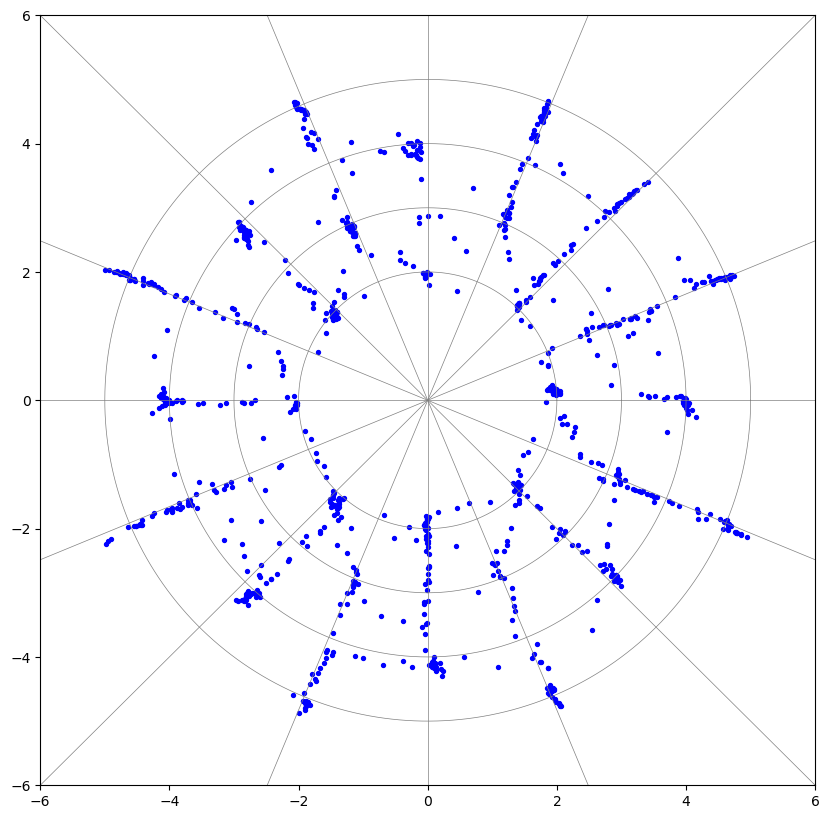}    
	\end{minipage}}
	\subfigure[Inclusive-NRF generation]{  
		\begin{minipage}{0.23\textwidth}  
			\flushleft  
			\includegraphics[width=\textwidth]{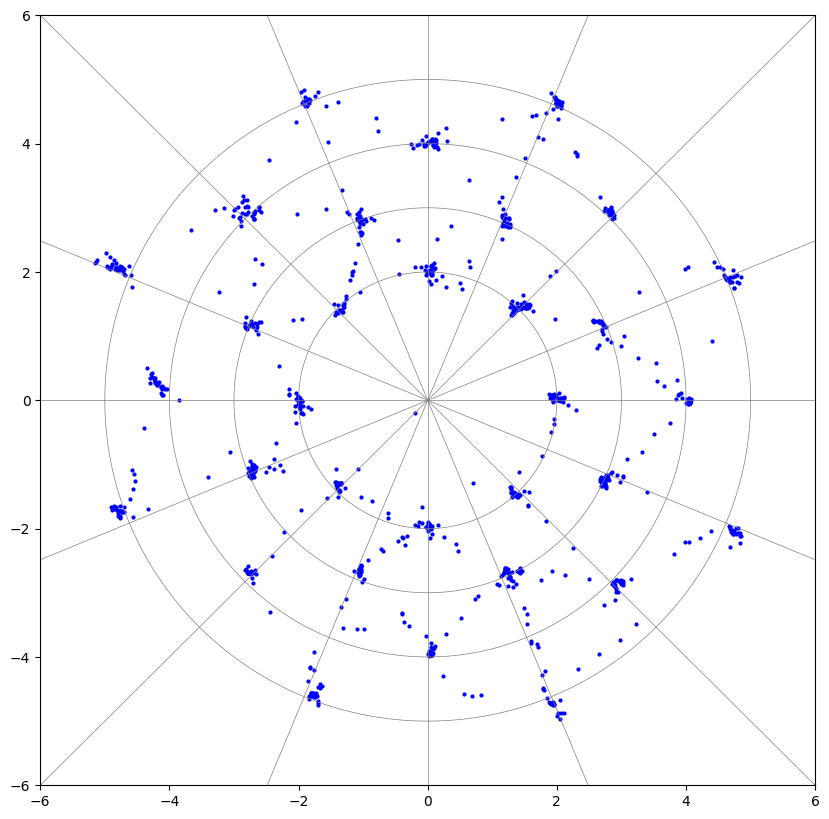}    
	\end{minipage}}
	%\hspace{3pt}
	\subfigure[Inclusive-NRF revision]{  
		\begin{minipage}{0.23\textwidth}  
			\centering  
			\includegraphics[width=\textwidth]{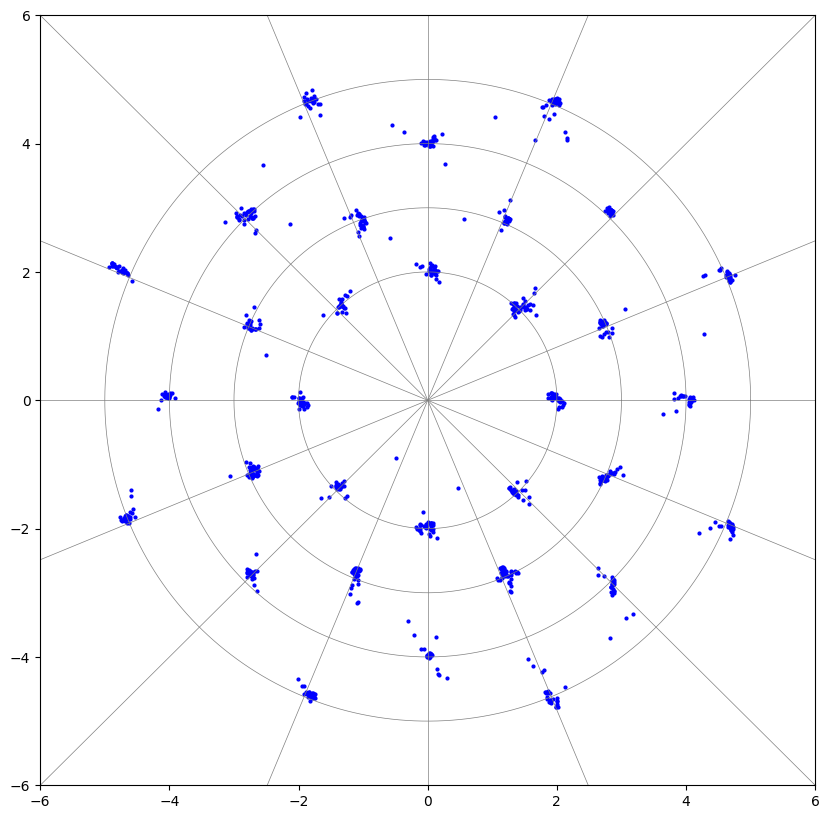}    
	\end{minipage}}
	\hspace{4pt}
	\subfigure[Exclusive-NRF potential]{  
		\begin{minipage}{0.24\textwidth}  
			\centering  
			\includegraphics[width=1\textwidth]{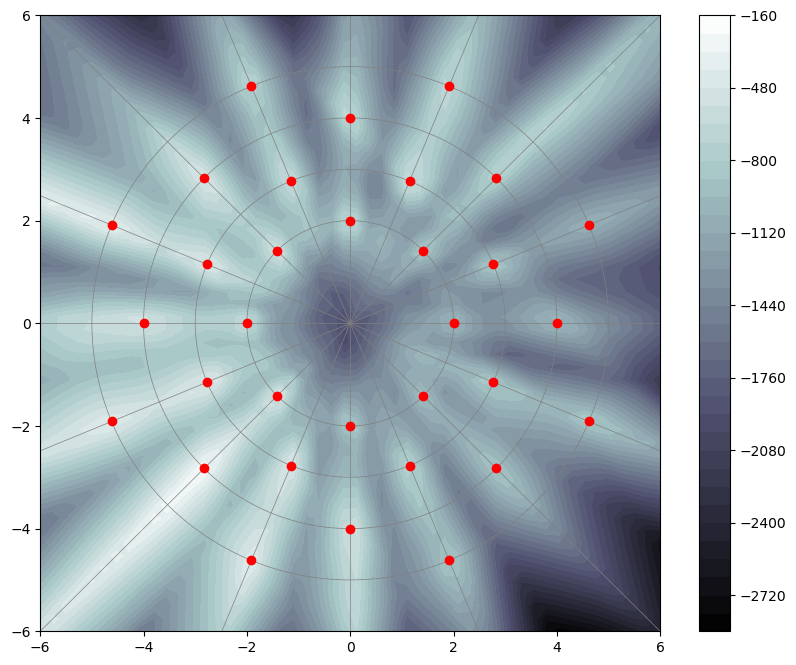}    
	\end{minipage}}
	%\hspace{4pt}
	\subfigure[Inclusive-NRF potential]{  
		\begin{minipage}{0.24\textwidth}  
			\centering  
			\includegraphics[width=1\textwidth]{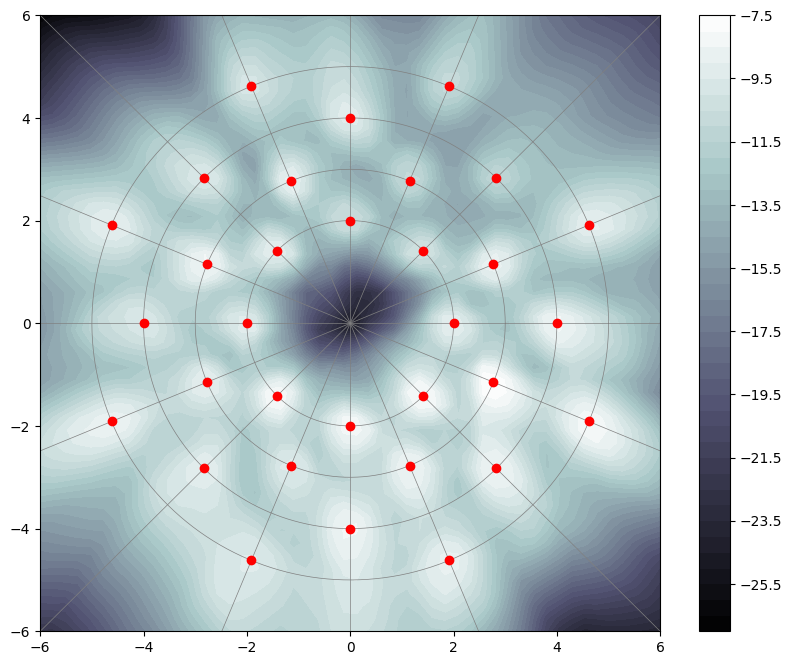}    
	\end{minipage}}
	\vskip -0.1in
	\caption{Comparison of different generative modeling methods over GMM synthetic data. 
		Stochastic generations from GAN with logD trick, WGAN-GP, Exclusive-NRF, \textbf{Inclusive-NRF generation} (i.e. sampling from the auxiliary generator) and \textbf{Inclusive-NRF revision} (i.e. after sample revision), are shown in (b)-(f) respectively.
		Each generation contains 1,000 samples.
		The learned potentials $u_\theta(x)$ from exclusive and inclusive NRFs are shown in (g) and (h) respectively, where the red dots indicate the mean of each Gaussian component.
		Inclusive NRFs are clearly superior in learning data density and sample generation.
	}
	\vskip -0.0in
	\label{fig:toy}
\end{figure*}

On the other hand, there are interesting connections between inclusive-NRFs and GANs, as elaborated in Appendix \ref{sec:connection-NRF-GAN}.
When interpreting the potential function $u_\theta(x)$ as the critic in Wasserstein GANs \cite{arjovsky2017wasserstein}, inclusive-NRFs seem to be similar to Wasserstein GANs. A difference is that in optimizing $\theta$ in inclusive-NRFs, the generated samples are further revised by taking finite-step-gradient of $u_\theta(x)$ w.r.t. $x$.
However, the critic in Wasserstein GANs can hardly be interpreted as an unnormalized log-density. Thus strictly speaking, inclusive-NRFs are not GAN-like.
From a different perspective, another interesting connection between kernel exponential family and MMD-GANs \cite{li2017mmd} is found in \cite{altun2006unifying,dai2018kernel}.

\subsection{Anomaly detection}

To evaluate the performance of generative models in density estimate, anomaly detection is a good real-world benchmarking task.
Anomaly detection (also known as one-class classification \cite{OC-SVM}) is a fundamental problem in machine learning, with critical applications in many areas, such as cybersecurity, complex system management, medical care, and so on. At the core of anomaly detection is density estimation: given a lot of input samples, anomalies are those ones residing in low probability density areas.

Anomaly detection has been extensively studied, as reviewed in recent works \cite{OC-SVM,SVDD,DSVDD,DSEBM,DAGMM,ALAD}.
Classical anomaly detection methods are kernel-based, e.g. One-Class Support Vecter Machine (OC-SVM) \cite{OC-SVM} and Support Vector Data Description (SVDD) \cite{SVDD}.
Such shallow methods typically require substantial feature engineering and also limited by poor computational scalability.
Recent methods leverage feature learning by using deep neural networks. Deep SVDD (DSVDD) \cite{DSVDD} combines a deep neural network with kernel-based SVDD.
In Deep Structured Energy-based Model (DSEBM) (DSEBM) \cite{DSEBM}, the models are essentially NRFs, but the training method is score matching \cite{hyvarinen2005estimation}.
Deep Autoencoding Gaussian Mixture Model (DAGMM) \cite{DAGMM} jointly train a deep autoencoder (which generates low-dimensional features) and a GMM (which operates on those low-dimensional features).
Adversarially Learned Anomaly Detection (ALAD) \cite{ALAD} uses a bi-directional GAN which needs one more network (the inference network) in addition to the generator and the discriminator networks.
In practice, the detection is usually performed by thresholding reconstruction errors (as used in DSEBM, ALAD) or density estimates (as used in DSEBM, DAGMM).
Both criteria are tested for DSEBM, denoted by DSEBM-r (reconstruction) and DSEBM-e (energy).
It is found that the energy score is a more accurate decision criterion than the reconstruction error \cite{DSEBM}.

	\section{Experiments} \label{sec:experiments} 
	
First, we report experimental result on synthetic dataset, which helps to illustrate different generative models and learning methods.
Then, extensive experiments are conducted to evaluate the performances of our approach (inclusive-NRFs) and various existing methods for image generation and anomaly detection on real-world datasets.
We refer to Appendix \ref{sec:exp-details} for experimental details and additional results.

	\begin{table}\normalsize
	\centering
	\caption{Numerical evaluations over the GMM (32 components) synthetic data. 
		The ``{covered modes}'' metric is defined as the number of covered modes by a set of generated samples.
		The ``{realistic ratio}'' metric is defined as the proportion of generated samples which are close to a mode.
		The measurement details are presented in text.
		Mean and SD are from 10 independent runs.
		{"Oracle" denotes the method of just drawing samples from the training data, which can be viewed as the top-line.}
	}
	\label{tab:toy}
	\begin{tabular}{lcc}
		\toprule
		Methods &covered modes  &realistic ratio \\
		\midrule
		GAN with logD trick \cite{goodfellow2014generative} &$21.47\pm1.44$ &$0.83\pm0.02$\\
		WGAN-GP \cite{Gulrajani2017ImprovedTO}&$22.21\pm1.75$ &$0.48\pm0.06$\\
		Exclusive-NRF \cite{Kim2016DeepDG}&$23.73\pm1.13$ &$0.50\pm0.03$\\  
		\textbf{Inclusive-NRF generation}    &$27.46\pm1.34$ &$0.67\pm0.08$\\
		\textbf{Inclusive-NRF revision} &$30.65\pm0.04$ &$0.95\pm0.01$\\
		Oracle&$30.70\pm0.09$ &$0.99\pm0.00$\\
		\bottomrule
	\end{tabular}
\end{table}
	
	\subsection{GMM synthetic experiment}\label{sec:GMM}
	
	The synthetic data consist of 1,600 training examples generated from a 2D Gaussian mixture model (GMM) with 32 equally-weighted, low-variance ($\sigma=0.1$) Gaussian components, uniformly laid out on four concentric circles as in Figure \ref{fig:toy}(a). 
	The data distribution exhibits many modes separated by large low-probability regions, which makes it suitable to examine how well different learning methods can deal with multiple modes.
	For comparison, we experiment with GAN with logD trick \cite{goodfellow2014generative} and  WGAN-GP \cite{Gulrajani2017ImprovedTO} for directed generative model, exclusive-NRF \cite{Kim2016DeepDG}, inclusive-NRF and CD for undirected generative model.

	The neural network architectures and hyperparameters for different methods are the same, as listed in Table \ref{tab:gmm-detail} in Appendix.
	We use SGLD \cite{sgld} for inclusive-NRFs on this synthetic dataset, with sample revision steps $L=10$ and empirical revision hyperparameters $\delta_l=0.01$.
	
	Figure \ref{fig:toy}(b)-(f) visually shows the generated samples from the trained models using different methods.
	For NRFs trained with different methods, i.e. the exclusive-NRF method, the inclusive-NRF method and the CD method, we also show the learned NRF potential $u_\theta(x)$ in Figure \ref{fig:toy}(g)(f) and Figure \ref{fig:CD-potential-map} (in Appendix) respectively.	
	Table \ref{tab:toy} reports the ``covered modes'' and ``realistic ratio'' as numerical measures of how the multi-modal data are fitted, 
	similarly as in \cite{dumoulin2016adversarially}.
	For all the methods in Table \ref{tab:toy} (including ``Oracle'', which denotes the method of just drawing samples from the training data), we use the following same procedure to estimate the metrics ``covered modes'' and ``realistic ratio''.
	\begin{enumerate}
		\item Stochastically generate 100 samples.
		
		\item A mode is defined to be covered (not missed) if there exist {at least one} generated sample located closely to the mode (with distance {$<3\sigma$}), and those samples are said to be realistic.
				
		\item Count how many modes are covered and calculate the proportion of realistic samples.
		
		\item Repeat the above steps 100 times and perform averaging.
	\end{enumerate}

	For each method, we independently train 10 models and calculate the mean and standard deviation (SD) across the 10 independent runs.
	For the ``Oracle'' method, we do not train models, but independently run the above procedure for 10 times.
	The main observations are as follows:
	%First, GANs are troubled in mode missing problem while NRF-IAGs, EBGMs and WGAN-GP perform well in balancing mode missing and mode covering (generating both diverse and realistic samples).
\begin{itemize}
\item GAN suffers from mode missing, generating realistic but not diverse samples.
	WGAN-GP increases ``covered modes'' but decreases ``realistic ratio''.
	Inclusive-NRF performs much better than both GAN and WGAN-GP in sample generation.
\item Inclusive-NRF outperforms exclusive-NRF in both sample generation and density estimation.
\item After revision, samples from inclusive-NRF become more like real samples, achieving the best in both ``covered modes'' and ``realistic ratio'' metrics.
\item As shown in Figure \ref{fig:CD-potential-map} in Appendix, the learned potential from applying the CD method to train the NRF yields a poor estimate of the data density, much worse than our inclusive-NRF approach.
\end{itemize}
	%In summary, NRF-IAGs perform remarkable in this synthetic unsupervised learning task, and we further apply NRF-IAGs to real-world datasets.

Remarkably, the superior performance of our inclusive-NRF approach in sample generation and density estimate in this synthetic experiment with many separated modes is a manifest of the efficiency in mixing between separate modes of our stochastic gradient sampler, which embodies both auxiliary variable MCMC and adaptive MCMC.
	
	\begin{table}
	\centering
	\caption{
		Inception score (IS) \cite{imporveGAN} and Frechet inception distance (FID) \cite{fid}
		results for image generation on CIFAR-10.
		"-" means the results are not reported in the original work.
On CIFAR-10, there are two widely benchmarked neural network architectures, referred to as CNNs \cite{cnn} and ResNets \cite{resnet}, which are grouped in the upper and lower blocks respectively.
The network architectures used by the methods in the same block are close. see Appendix \ref{sec:detail-cifar-10} for experimental details.
	}
	\label{is_unsupervised}
	\begin{tabular}{l|lcc}
		\toprule
		&Methods&IS$\uparrow$&FID$\downarrow$\\
		\midrule
		\multirow{4}*{CNNs}&DCGAN \cite{radford2015unsupervised}&$6.16\pm0.07$&-\\
		&CoopNet \cite{coopnets}&-&$33.61$\\
		&SNGAN \cite{Miyato2018SpectralNF}&$7.58\pm0.12$&25.5\\
		&\textbf{Inclusive-NRF generation} &$7.54\pm0.10$&$27.9\pm0.53$\\
		\midrule
			\multirow{7}*{ResNets}&WGAN-GP \cite{Gulrajani2017ImprovedTO} &$7.86\pm0.07$&-\\
			&CT-GAN \cite{ct-gan}&$8.12\pm0.12$&-\\
			&Fisher-GAN \cite{fisher-gan}&$7.90\pm0.05$&-\\
			&BWGAN \cite{bwgan}&$8.26\pm0.07$&-\\
			&IGG-EBM \cite{IGG-EBM} &6.78&38.2\\
		&SNGAN \cite{Miyato2018SpectralNF}&$8.22\pm0.05$&$21.7\pm0.21$\\
		&\textbf{Inclusive-NRF generation} &$8.28\pm0.09$&$20.9\pm0.25$\\
		\bottomrule
	\end{tabular}
\end{table}

	\subsection{Image generation on CIFAR-10}
	
	In this experiment, we evaluate the performance of different generative models in image generation over the widely used real-world dataset CIFAR-10 \cite{krizhevsky2009learning}.
	Apart from visual inspection of generated samples, we use inception score (IS) \cite{imporveGAN} (the larger the better) and Frechet inception distance (FID) \cite{fid} (the smaller the better)\footnote{IS \cite{imporveGAN} is based on the classification output $p(y|x)$ from the Inception model \cite{szegedy2016rethinking}. Defined as $\exp \left[ E_x KL[p(y|x) || p(y)] \right] $, IS is highest when each image's predictive distribution has low entropy (measuring realistic), and the marginal predictive distribution $p(y)=E_x p(y|x)$ has high entropy (measuring ``mode covering'' or say diversity). IS correlates somewhat with human judgement of sample quality on natural images.
		FID \cite{fid} fits a Gaussian distribution to the samples' representations (the 2048-dimensional hidden activations at the $pool3$ layer in the Inception model) for the data samples and generated samples respectively, and then calculates the Frechet distance between the two Gaussians.} to quantitatively evaluate the generation quality.
Evaluation of generative models is non-trivial and it is important that evaluation metrics matches the target applications \cite{theis2016a}.
	We believe that the combination of IS and FID can serve for this evaluation purpose in image generation, and can reflect the generation performance of different models in ``covered modes'' and ``realistic ratio'' metrics.
	
	%which takes into account that the samples are both diverse and realistic and is usually consistent with human judgment.
	%Table \ref{is_unsupervised} reports the inception scores for various methods, which are calculated using a widely used classification model from \cite{imporveGAN}.
	
	Table \ref{is_unsupervised} reports the IS and FID results for various methods.
	On CIFAR-10, there are two widely benchmarked neural network architectures, referred to as CNNs \cite{cnn} and ResNets \cite{resnet}. ResNets are larger than CNNs.
	To isolate the effect of network architectures on the performance by different methods, we group the results in two corresponding blocks.
	The network architectures used by the methods in the same block are close. See Appendix \ref{sec:detail-cifar-10} for experimental details.
	
	From the results in Table \ref{is_unsupervised}, it can be seen that the proposed inclusive-NRF models achieve state-of-the-art performance in both settings of CNNs and ResNets, and even outperform SNGAN \cite{Miyato2018SpectralNF} in the setting of using large ResNets.
	Remarkably, compared to both the recent IGG-EBM \cite{IGG-EBM} which learns NRFs without auxiliary models and the recent CoopNet \cite{coopnets} which learns NRFs with auxiliary generator, inclusive-NRFs obtain significantly better results.
	The superiority of inclusie-NRF models in sample generation in GMM synthetic experiment is further substantiated in real-world image generation.
Some generated samples are shown in Figure \ref{fig:generate} in Appendix.
	We also show the capability of inclusive-NRFs in latent space interpolation (Appendix \ref{sec:latent-space-interp}).

It is worthwhile to comment on the time complexity of different training methods.
Compared to training GANs, the extra cost for training inclusive-NRFs involves running SGLD/SGHMC for $L$ steps to obtain samples.
Once the samples are obtained, the cost of parameter updating is the same as in GANs.
For the results in Table \ref{is_unsupervised}, we use SGLD with $L=1$ and spectral normalization, which is a good trade-off between computation cost and model performance.
Specifically, when evaluated in training on CIFAR-10 with ResNets architecture, the wall-clock training time is $22$ hours for SNGAN \cite{Miyato2018SpectralNF} on $1$ Nvidia P100 GPU, $35$ hours for inclusive-NRF on $1$ Nvidia P100 GPU, and $160$ hours for IGG-EBM \cite{IGG-EBM} on $4$ Nvidia P100 GPUs.
	
	%It can been seen from Table \ref{is_unsupervised} that NRF-IAGs outperforms DCGAN \cite{radford2015unsupervised}, WGAN-GP \cite{Gulrajani2017ImprovedTO}, ALI \cite{dumoulin2016adversarially}, Improved-GAN \cite{imporveGAN} with a large margin, and performs slightly better than D2GAN \cite{Nguyen2017DualDG} (3 networks).
	%The higher IS from DFM \cite{WardeFarley2017ImprovingGA} uses more complicated structure (3 networks) than NRF-IAGs (2 networks).
	%In Table \ref{is_unsupervised} and Figure \ref{fig:generate}(c)(d), we compare the performances of NRF-IAGs with SGLD and SGHMC.
	
	%Utilizing SGHMC in NRF-IAGs to exploit gradient information with momentum yields better performance than simple SGLD as used in \cite{coopnets}. 

\begin{table}
	\centering
	\caption{Ablation study of our inclusive-NRF method on CIFAR-10, regarding the effects of using SGLD or SGHMC in training.
		Mean and SD are from 5 independent runs for each training setting.
		In each training setting, two manners to generate samples given a trained NRF are compared, as previously illustrated in Figure \ref{fig:toy} over synthetic GMM data.
		We examine generated samples (i.e. directly from the generator) and revised samples (i.e. after sample revision) respectively, in term of inception scores (IS).		
	}
	\label{tab:ablation}
	\begin{tabular}{lcc}
		\toprule
		Training Setting&Generation IS&Revision IS\\
		\midrule
		SGLD $L=2$&$7.38\pm0.07$&$7.43\pm0.06$\\
		SGLD $L=5$&$7.45\pm0.16$&$7.47\pm0.10$\\
		SGLD $L=10$&$7.45\pm0.13$&$7.48\pm0.12$\\
		SGHMC $L=10$&$7.47\pm0.08$&$7.56\pm0.06$\\
		\bottomrule
	\end{tabular}
\end{table}

\subsubsection{Ablation study}
\label{sec:ablation}

We report the results of ablation study of our inclusive-NRF method on CIFAR-10 in Table \ref{tab:ablation}.
In this experiment, we use the CNN network architecture with batch normalization. See Appendix \ref{sec:detail_ablation} for experimental details.
We analyze the effects of different settings in model training, such as using SGLD or SGHMC and the revision step $L=2/5/10$ used. For each training setting, we also compare the two manners to generate samples - whether applying sample revision or not in inference (i.e. generating samples) given a trained NRF, as previously illustrated in Figure \ref{fig:toy} over synthetic GMM data.
The main observations are as follows.

%Firstly, by employs a revison process to obtain approximate samples of random fields, the semi-supervised performance is improved while the unsupervised performance is not benefited.
%It is probably because that with a revision process, the random fields are improved and the classification performance is totally depended on random fields, thus classificaition error is decreased.
%With a revision process, the generators do not get gradients from the random fields (like the discrminator in GANs) directly, so generators may not be improved.
First, given a trained inclusive-NRF, after revision (i.e. following the gradient of the potential $u_\theta(x)$ w.r.t. $x$), the quality (measured by IS) of samples is always improved, as shown by the consistent IS improvement from the second column (generation) to the third (revision).
This is in accordance with the results in the GMM synthetic experiments in Section \ref{sec:GMM}.
Note that in revision, it is the estimated density $p_\theta$ that guides the samples towards low energy region of the random field. Thus, the improved sample quality after revision demonstrates the capability of our inclusive-NRFs in density estimate from training data.

Second, a row-wise reading of Table \ref{tab:ablation} shows that using SGLD with more revision steps in training consistently improve the sample quality.
Further using SGHMC in training to exploit gradient information with momentum yields better performance than simple SGLD.
These results reveal the benefit of our new development of stochastic gradient samplers, especially SGHMC.
In our experiments, it is also found that when using spectral normalization, the effect of increasing revision steps and using momentum on model performance is less significant.
Presumably, spectral normalization may obscure the true learning dynamics as discussed in \cite{nijkamp2019anatomy}\footnote{Note that the Lipschitz norm of the potential $u_\theta(x)$ represents the maximum gradient strength and plays a central role in stochastic gradient sampling through the term $\delta_l \Delta (z^{(l-1)};\lambda)$ as shown in Eq. (\ref{eq:SGLD})(\ref{eq:SGHMC}).
	And the Lipschitz norm of the potential is related to the network spectral norm.
	Thus, the stochastic gradient sampler may be disturbed by the external spectral normalization technique.}.

%It is also found that more revision steps in model training do not significantly improve unsupervised IS. So we can use $L=1$ in unsupervised learning for generation, which can reduce the computational cost.

		\begin{table*}\normalsize
	\centering
	\caption{Anomaly detection results on KDDCUP dataset.
		Results for OC-SVM, DSEBM-e, DAGMM are taken from \cite{DAGMM}.
		Inclusive-NRF results are obtained from 20 runs, each with random split of training and test sets, as in \cite{DAGMM}.
		ALAD result uses a fixed split with random parameter initializations, and thus has smaller standard deviations.
	}
	\label{tab:ad_kddcup}
	\begin{tabular}{cccc}
		\toprule
		Model&Precision&Recall&F1\\
		\midrule
		OC-SVM \cite{OC-SVM}&0.7457&0.8523&0.7954\\
		DSEBM-e \cite{DSEBM}&0.7369&0.7477&0.7423\\
		DAGMM \cite{DAGMM}&0.9297&0.9442&0.9369\\
		ALAD \cite{ALAD}&$0.9427\pm0.0018$&$0.9577\pm0.0018$&$0.9501\pm0.0018$\\
		Inclusive-NRF SGLD&$0.9452\pm0.0105$&$0.9600\pm0.0113$&$0.9525\pm0.0108$\\
		Inclusive-NRF SGHMC&$\mathbf{0.9501\pm0.0043}$&$\mathbf{0.9651\pm0.0062}$&$\mathbf{0.9575\pm0.0050}$\\
		\bottomrule
	\end{tabular}
\end{table*}

\begin{table*}\normalsize
	\centering
	\caption{Anomaly detection results on MNIST and CIFAR-10 measured by AUCs (\%).
		Both datasets have ten different classes from which we create ten one-class classification setups.
		Results for each one-class setup are obtained from 10 runs (with random parameter initializations), and averaging the mean AUCs over the ten setups gives the ``mean'' result.
		Results for OC-SVM, KDE (Kernel density estimation), IF (Isolation Forest), DCAE (Deep Convolutional AutoEncoders), AnoGAN, DSVDD are taken from \cite{DSVDD}.
	}
	\label{tab:ad_mnist_cifar}
	\begin{tabular}{lcccccccc}
		\toprule
		\multirow{2}*{Normal Class}&\multirow{2}*{OC-SVM}&\multirow{2}*{KDE}&\multirow{2}*{IF}&\multirow{2}*{DCAE}&\multirow{2}*{AnoGAN}&\multirow{2}*{DSVDD}&\multicolumn{2}{c}{Inclusive-NRF}\\
		&&&&&&&SGLD&SGHMC\\
		\midrule
		Digit 0&98.6$\pm$0.0&97.1$\pm$0.0&98.0$\pm$0.3&97.6$\pm$0.7&96.6$\pm$1.3&98.0$\pm$0.7&\textbf{98.9$\pm$0.7}&\textbf{98.9$\pm$0.6}\\
		Digit 1&99.5$\pm$0.0&98.9$\pm$0.0&97.3$\pm$0.4&98.3$\pm$0.6&99.2$\pm$0.6&99.7$\pm$0.1&\textbf{99.8$\pm$0.1}&\textbf{99.8$\pm$0.1}\\
		Digit 2&82.5$\pm$0.1&79.0$\pm$0.0&88.6$\pm$0.5&85.4$\pm$2.4&85.0$\pm$2.9&91.7$\pm$0.8&\textbf{93.2$\pm$3.2}&91.8$\pm$4.0\\
		Digit 3&88.1$\pm$0.0&86.2$\pm$0.0&89.9$\pm$0.4&86.7$\pm$0.9&88.7$\pm$2.1&91.9$\pm$1.5&93.4$\pm$2.0&\textbf{93.8$\pm$2.6}\\
		Digit 4&94.9$\pm$0.0&87.9$\pm$0.0&92.7$\pm$0.6&86.5$\pm$2.0&89.4$\pm$1.3&94.9$\pm$0.8&\textbf{95.6$\pm$1.2}&\textbf{95.6$\pm$1.9}\\
		Digit 5&77.1$\pm$0.0&73.8$\pm$0.0&85.5$\pm$0.8&78.2$\pm$2.7&88.3$\pm$2.9&88.5$\pm$0.9&89.4$\pm$6.5&\textbf{94.9$\pm$1.4}\\
		Digit 6&96.5$\pm$0.0&87.6$\pm$0.0&95.6$\pm$0.3&94.6$\pm$0.5&94.7$\pm$2.7&98.3$\pm$0.5&\textbf{98.6$\pm$0.5}&97.5$\pm$2.7\\
		Digit 7&93.7$\pm$0.0&91.4$\pm$0.0&92.0$\pm$0.4&92.3$\pm$1.0&93.5$\pm$1.8&94.6$\pm$0.9&96.1$\pm$1.1&\textbf{96.4$\pm$1.0}\\
		Digit 8&88.9$\pm$0.0&79.2$\pm$0.0&89.9$\pm$0.4&86.5$\pm$1.6&84.9$\pm$2.1&\textbf{93.9$\pm$1.6}&85.1$\pm$6.7&88.9$\pm$3.3\\
		Digit 9&93.1$\pm$0.0&88.2$\pm$0.0&93.5$\pm$0.3&90.4$\pm$1.8&92.4$\pm$1.1&\textbf{96.5$\pm$0.3}&93.8$\pm$0.8&94.9$\pm$1.0\\
		\textbf{Mean}&91.29&86.93&92.30&89.65&91.27&94.80&94.38&\textbf{95.26}\\
		\midrule
		AIRPLANE&61.6$\pm$0.9&61.2$\pm$0.0&60.1$\pm$0.7&59.1$\pm$5.1&67.1$\pm$2.5&61.7$\pm$4.1&78.1$\pm$2.1&\textbf{80.3$\pm$2.8}\\
		AUTOMOBILE&63.8$\pm$0.6&64.0$\pm$0.0&50.8$\pm$0.6&57.4$\pm$2.9&54.7$\pm$3.4&65.9$\pm$2.1&71.6$\pm$2.1&\textbf{87.0$\pm$1.8}\\
		BIRD&50.0$\pm$0.5&50.1$\pm$0.0&49.2$\pm$0.4&48.9$\pm$2.4&52.9$\pm$3.0&50.8$\pm$0.8&65.4$\pm$1.8&\textbf{68.9$\pm$2.0}\\
		CAT&55.9$\pm$1.3&56.4$\pm$0.0&55.1$\pm$0.4&58.4$\pm$1.2&54.5$\pm$1.9&59.1$\pm$1.4&63.3$\pm$1.9&\textbf{64.9$\pm$2.5}\\
		DEER&66.0$\pm$0.7&66.2$\pm$0.0&49.8$\pm$0.4&54.0$\pm$1.3&65.1$\pm$3.2&60.9$\pm$1.1&70.5$\pm$2.2&\textbf{72.7$\pm$1.8}\\
		DOG&62.4$\pm$0.8&62.4$\pm$0.0&58.5$\pm$0.4&62.2$\pm$1.8&60.3$\pm$2.6&65.7$\pm$2.5&64.1$\pm$2.7&\textbf{66.9$\pm$4.0}\\
		FROG&74.7$\pm$0.3&74.9$\pm$0.0&42.9$\pm$0.6&51.2$\pm$5.2&58.5$\pm$1.4&67.7$\pm$2.6&75.4$\pm$2.4&\textbf{76.2$\pm$2.3}\\
		HORSE&62.6$\pm$0.6&62.6$\pm$0.0&55.1$\pm$0.7&58.6$\pm$2.9&62.5$\pm$0.8&67.3$\pm$0.9&66.1$\pm$3.7&\textbf{75.4$\pm$3.1}\\
		SHIP&74.9$\pm$0.4&75.1$\pm$0.0&74.2$\pm$0.6&76.8$\pm$1.4&75.8$\pm$4.1&75.9$\pm$1.2&75.6$\pm$2.8&\textbf{78.7$\pm$1.7}\\
		TRUCK&75.9$\pm$0.3&76.0$\pm$0.0&58.9$\pm$0.7&67.3$\pm$3.0&66.5$\pm$2.8&73.1$\pm$1.2&70.1$\pm$2.0&\textbf{78.8$\pm$2.1}\\
		\textbf{Mean}&64.78&64.89&55.46&59.39&61.79&64.81&70.02&\textbf{74.98}\\
		\bottomrule
	\end{tabular}
\end{table*}

\subsection{Anomaly detection}

To evaluate the performance of different generative models in density estimate, anomaly detection is a good real-world benchmarking task, in addition to the GMM synthetic experiments (Section \ref{sec:GMM}).

In applying inclusive-NRFs to anomaly detection, the unnormalized density estimates (as measured by potential values) provide a natural decision criterion, since the normalizing constant only introduces a constant in thresholding.
Specifically, after training inclusive-NRFs on data containing only the samples of the normal class, testing samples with potential values lower than a threshold are detected as anomaly.
We evaluated our inclusive-NRFs for anomaly detection on publicly available tabular and image datasets - KDDCUP, MNIST and CIFAR-10. See Appendix \ref{sec:detail_anomaly} for experimental details.

	\textbf{KDDCUP.} For tabular data, we test on the KDDCUP99 ten percent dataset \cite{KDDCUP} (denoted as KDDCUP). 
	This dataset is a network intrusion dataset, originally contains samples of 41 dimensions, where 34 of them are continuous and 7 are categorical.
	One-hot representation is used to encode categorical features, and eventually a dataset of 120 dimensions is obtained. 
	As 20\% of data samples are labeled as ``normal'' and the rest are labeled as ``attack'', ``normal'' ones are thus treated as anomalies in this task.
	For each run, we randomly take 50\% of the whole dataset and use only data samples from the normal class for training models; the rest 50\% is reserved for testing.
	During testing, the 20\% samples with lowest potentials are marked as anomalies.
	The anomaly class is regarded as positive, and precision, recall, and F1 score are calculated accordingly.
Here we follow the setup in prior work \cite{OC-SVM,DSEBM,DAGMM,ALAD} so that the results are directly comparable.
	 
	From the results shown in Table \ref{tab:ad_kddcup}, it can be seen that inclusive-NRF outperforms all state-of-the-art methods (DSEBM-e, DAGMM, ALAD).
	Particularly, compared to the previous state-of-the-art deep energy model (DSEBM-e), inclusive-NRF outperforms by a large margin (above +0.2 F1 score), which clearly shows the superiority of our new approach in learning NRFs.
	The performance of inclusive-NRF with SGHMC is much better than with SGLD, which shows the superiority of our SGHMC sampler compared to Langevin dynamics used in prior work of learning NRFs \cite{coopnets}.
	
	\textbf{MNIST and CIFAR-10.} 
	Both datasets have ten different classes from which we create ten one-class classification setups.
	The standard training and test splits of MNIST and CIFAR-10 are used, and we follow the ``one-class'' setup in \cite{DSVDD}.
	For each setup, one of the classes is the normal class and samples from the remaining classes are treated as anomalies; only training samples from the normal class are employed for model training. 
	Accordingly, for each setup, training set sizes are 6000 for MNIST and 5000 for CIFAR-10, and the test set consists of 1000 normal samples and 9000 abnormal samples. 
	For comparison with existing results, AUC (area under the receiver operating curve) \cite{AUC} is employed as the performance metric in both datasets.
	From the results shown in Table \ref{tab:ad_mnist_cifar}, it can be seen that inclusive-NRF performs much better state-of-the-art method (DSVDD) which is specifically designed for anomaly detection.
	Again, learning inclusive-NRFs with SGHMC obtains better performance than with SGLD, especially on CIFAR-10 dataset which is more challenging than MNIST.

To sum up, here we show that a straightforward application of the inclusive-NRF approach in anomaly detection on both tabular and image datasets achieves superior performance over state-of-the-art methods.
	This is a clear indication of the ability of inclusive-NRFs for density estimation.
	
\section{Conclusion and future directions} \label{sec:discussion}

Neural random fields (NRFs), referring to a class of generative models that use neural networks to implement potential functions in random fields, are not new but receive less attention with slow progress, though with their own merits.
In this paper we propose a new approach, the inclusive-NRF approach, to learning NRFs for continuous data (e.g. images).
On top of prior work, the inclusive-NRF approach is the first in learning NRFs for continuous data by minimizing the inclusive-divergence and developing stochastic gradient sampling, with theoretical analysis.

Based on the new approach, specific inclusive-NRF models are developed and thoroughly evaluated in two important generative modeling applications - image generation and anomaly detection. The proposed models consistently achieve strong experimental results in both applications compared to state-of-the-art methods.
These superior performances presumably are attributed to the two distinctive features in inclusive-NRFs - introducing the inclusive-divergence minimized auxiliary generator and performing model sampling by SGLD/SGHMC in the augmented space.
The auxiliary generator acts like an adaptive proposal, updated by using SGLD/SGHMC samples from the extended target density. 
Minimizing the inclusive-divergence tends to drive the auxiliary generator to have higher entropy than the extended target density, which is a desirable property for proposal design in MCMC.
These features enable more efficient sampling and thus more successful learning of NRFs.

There are several worthwhile directions for future research. First, the flexibility of the inclusive-NRF approach is worth emphasizing.
We anticipate the application of the inclusive-NRF approach to more generative modeling tasks.
Second, although this work deals with fix-dimensional data, it is an important direction of extending the inclusive-NRF approach to sequential and trans-dimensional modeling tasks, such as speech, language, video and so on.

\section{Acknowledgments} \label{sec:acknowledgments}
This work was supported by NSFC Grant 61976122, China MOE-Mobile Grant MCM20170301. The authors would like to thank Zhiqiang Tan for helpful discussions.

	\bibliographystyle{IEEEtran}
	\bibliography{NRF-NNLS}
	
\begin{IEEEbiography}[{\includegraphics[width=1in,height=1.25in,clip,keepaspectratio]{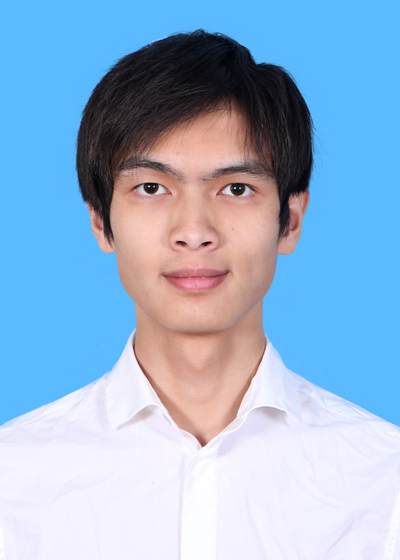}}]{Yunfu Song}
	received the B.S. degree in Physics from Tsinghua University in 2017.
	Since 2017, he has been pursuing a master degree at the Department of Electronic Engineering, Tsinghua University under the supervision of Zhijian Ou.
	His research focuses on learning with undirected graphical models and neural networks.
\end{IEEEbiography}
\vspace{-30pt}
\begin{IEEEbiography}[{\includegraphics[width=1in,height=1.25in,clip,keepaspectratio]{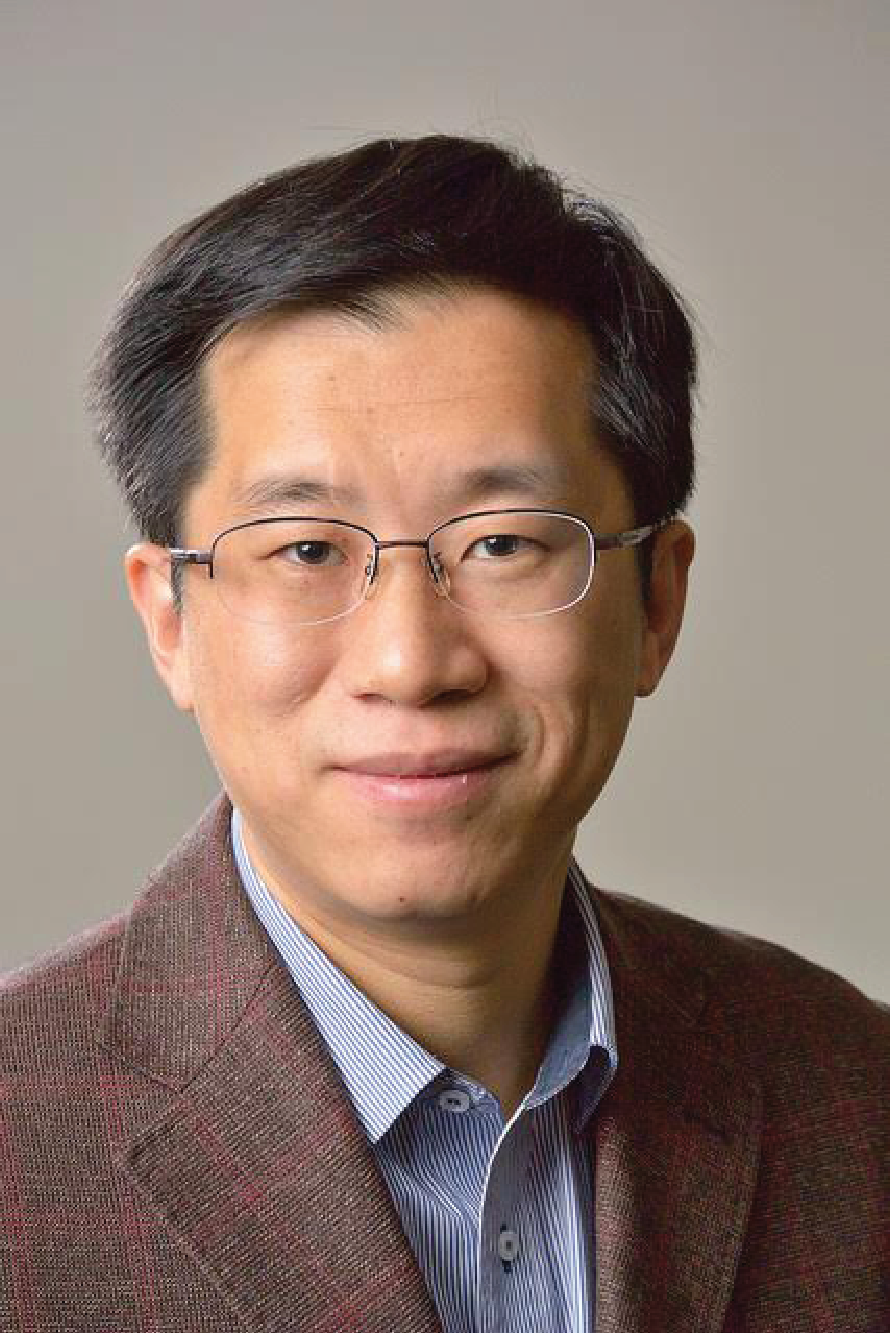}}]{Zhijian Ou}
	received the B.S. degree (with the highest honor) in electronic engineering from Shanghai Jiao Tong University in 1998 and the Ph.D. degree in electronic engineering from Tsinghua University in 2003.
	Since 2003, he has been with the Department of Electronic Engineering in Tsinghua University and is currently an associate professor. From August 2014 to July 2015, he was a visiting scholar at Beckman Institute, University of Illinois at Urbana-Champaign.
	He has been an associate editor of IEEE/ACM Transactions on Audio, Speech and Language Processing and a Member of IEEE Speech and Language Processing Technical Committee, since 2019.
	He has actively led research projects from NSFC, China 863 High-tech Research and Development Program, China Ministry of Information Industry and China Ministry of Education, as well as joint-research projects with Intel, Panasonic, IBM, and Toshiba.
	His recent research interests include speech and language processing (speech recognition and understanding, speaker recognition, natural language processing) and machine intelligence (particularly with graphical models).
\end{IEEEbiography}

	\normalsize

\clearpage

%\setcounter{page}{1}

%\setcounter{section}{0}
%\setcounter{equation}{0}
%\setcounter{figure}{0}
%\setcounter{table}{0}

%\renewcommand{\theequation}{S\arabic{equation}}
%\renewcommand\thefigure{S\arabic{figure}}
%\renewcommand\thetable{S\arabic{table}}

%% title

%	\hrule height 4pt
%	\vskip 0.25in
%	
%	\begin{center}
%		{\LARGE\bf Supplement for\\
%			``Learning Neural Random Fields with Inclusive Auxiliary Generators''} 
%	\end{center}
%	
%	\vskip 0.29in
%	\hrule height 1pt
%	\vskip 0.09in%

%\begin{center}
%	\begin{tabular}[t]{c}5
%		{\bf Yunfu Song, Zhijian Ou} \\
%		Tsinghua University, Beijing, China\\
%		\texttt{ozj@tsinghua.edu.cn}\\
%	\end{tabular}%
%\end{center}

	\appendices

\section{Proof of Proposition \ref{prop:nrf_gradient_proof}}
\label{sec:proof-prop-1}

\begin{proof}
	The first line of Eq.(\ref{eq:jrf_unsup_gradient}) can be obtained by directly taking derivative of $KL\left[  \tilde{p}(\tilde{x}) || p_\theta(\tilde{x}) \right]$ w.r.t. $\theta$, as shown below,
	\begin{displaymath}
	\begin{aligned}
	\frac{\partial}{\partial \theta} KL\left[  \tilde{p}(\tilde{x}) || p_\theta(\tilde{x}) \right] &= \frac{\partial}{\partial \theta} \int \tilde{p}(\tilde{x}) \log  \frac{\tilde{p}(\tilde{x})}{p_\theta(\tilde{x})} d \tilde{x} \\&= - \int \tilde{p}(\tilde{x}) \frac{\partial}{\partial \theta} \log p_\theta(\tilde{x}) d \tilde{x},
	\end{aligned}
	\end{displaymath}
	and then applying the basic formula of Eq. (\ref{eq:RF-grad}).
	
	For the second line, by direct calculation, we first have
	\begin{displaymath}
	\begin{aligned}
	&E_{q_\phi(h|x)}\left[ \nabla_\phi \log q_\phi(h|x) \right] \\=& \int q_\phi(h|x) q_\phi(h|x)^{-1} \nabla_\phi q_\phi(h|x) dh
	\\=& \int \nabla_\phi q_\phi(h|x) dh
	= \nabla_\phi \int q_\phi(h|x) dh
	= \nabla_\phi 1 = 0.
	\end{aligned}
	\end{displaymath}
	
	Then combining 
	\begin{displaymath}
	\frac{\partial}{\partial \phi} KL\left[  p_\theta(x) || q_\phi(x) \right] = - E_{p_\theta(x)}\left[ \nabla_\phi \log q_\phi(x)\right]
	\end{displaymath} 
	and
	\begin{displaymath}
	\begin{aligned}
	\nabla_\phi \log q_\phi(x)
	&=E_{q_\phi(h|x)}\left[ \nabla_\phi \log q_\phi(x)\right]
	\\&=E_{q_\phi(h|x)}\left[ \nabla_\phi \log q_\phi(x,h) - \nabla_\phi \log q_\phi(h|x) \right]\\
	&=E_{q_\phi(h|x)}\left[ \nabla_\phi \log q_\phi(x,h)\right].
	\end{aligned}
	\end{displaymath}
	will give the second line of Eq.(\ref{eq:jrf_unsup_gradient}).
\end{proof}

\section{Proof of Proposition \ref{prop:unbiased-est}}
\label{sec:proof-prop-2}
\begin{proof}
	\begin{displaymath}
	\begin{aligned}
	\frac{\partial}{\partial x} \log q_\phi(x)
	&=E_{q_\phi(h^*|x)}\left[ \frac{\partial}{\partial x} \log q_\phi(x)\right]
	\\&=E_{q_\phi(h^*|x)}\left[\frac{\partial}{\partial x} \log q_\phi(x,h^*) - \frac{\partial}{\partial x} \log q_\phi(h^*|x) \right]\\
	&=E_{q_\phi(h^*|x)}\left[ \frac{\partial}{\partial x} \log q_\phi(x,h^*)\right].
	\end{aligned}
	\end{displaymath}
\end{proof}
		
	\section{Proof of Proposition \ref{prop:Kim-Bengio} (Exclusive-NRF)}
	\label{sec:proof-kim-bengio}
	
	\begin{prop} \label{prop:Kim-Bengio}
		For the RF as defined in Eq. (\ref{eq:unsup-RF}), we have the following evidence upper bound:  
		\begin{equation} \label{eq:exclusive-bound}
		\begin{aligned}
		log p_\theta(\tilde{x}) &= \mathcal{U}(\tilde{x};\theta,\phi) - KL(q_\phi(x)||p_\theta(x)) \\
		&\le \mathcal{U}(\tilde{x};\theta,\phi),
		\end{aligned}
		\end{equation}
	where $\mathcal{U}(\tilde{x};\theta,\phi) \triangleq u_\theta(\tilde{x}) - \left(  E_{q_\phi(x)}\left[ u_\theta(x)\right]
	+ H \left[ q_\phi(x) \right] \right).$	
	\end{prop}
	\begin{proof}
		Note that (i) $	log p_\theta(\tilde{x}) = u_\theta(\tilde{x}) - log Z(\theta)$, and (ii)
		we have the following lower bound on $Z(\theta)$: 
		$log Z(\theta) = log \int\exp(u_{\theta}(x))dx = log \int q_\phi(x) \frac{\exp(u_{\theta}(x))}{q_\phi(x)}dx \ge \int q_\phi(x) log \frac{\exp(u_{\theta}(x))}{q_\phi(x)}dx $. Combining (i) and (ii) gives Eq. (\ref{eq:exclusive-bound}).
	\end{proof}
	Furthermore, it can be seen that learning in \cite{Kim2016DeepDG} amounts to optimizing the following evidence upper bound:
	\begin{displaymath}
	\max_{\theta} \min_{\phi} \mathcal{U}(\tilde{x};\theta,\phi),
	\end{displaymath}
	which is unfortunately not revealed in this manner in \cite{Kim2016DeepDG}.
	
	\section{Connection between inclusive-NRFs and GANs}
	\label{sec:connection-NRF-GAN}
	
	Note that for the generator as defined in Eq. (\ref{eq:generator}), we have the following joint density
	\begin{displaymath}
	logq_\phi(x,h) = - \frac{1}{2 \sigma^2} || x - g_\phi(h) ||^2 + constant.
	\end{displaymath}
	The generator parameter $\phi$ is updated according to the gradient in Eq. (\ref{eq:jrf_unsup_gradient}), which is rewritten as follows:
	\begin{displaymath}
	\label{eq:grad_phi}
	\nabla_\phi = E_{p_\theta(x) q_\phi(h|x)}\left[ \nabla_\phi  logq_\phi(x,h)\right]
	\end{displaymath}
	
	Specifically, we draw $(h',x') \sim q_\phi$ and then perform one-step SGLD to obtain $(h,x)$. To simply the analysis of the connection, suppose $h \approx h'$, $x' \approx g_\phi(h') \approx g_\phi(h) $. Then we have
	\begin{equation}\label{eq:SGLD-one-step}
	\begin{split}
&x = x'+\delta_1 \left.\left[ \frac{\partial}{\partial x} log p_\theta(x) \right] \right|_{x=x'}   +\sqrt{2\delta_1} \eta^{(1)},\\
&\quad \eta^{(1)} \sim \mathcal{N}(0,I)
	\end{split}
	\end{equation}
which further gives
	\begin{displaymath} 
	\begin{split}
	&x - g_\phi(h) \approx \delta_1 \left.\left[ \frac{\partial}{\partial x} log p_\theta(x) \right] \right| _{x=g_\phi(h)} \\
	=& \delta_1 \left.\left[ \frac{\partial}{\partial x} u_\theta(x) \right] \right| _{x=g_\phi(h)} 
	\end{split}
	\end{displaymath}
	
	The gradient in the updating step in Algorithm \ref{alg:learning-NRF-IAG} becomes:
	\begin{displaymath}
	\begin{split}
	\nabla_\phi  logq_\phi(x,h) &= \frac{1}{\sigma^2} \left[ \frac{\partial}{\partial \phi} g_\phi(h) \right]  \left[ x - g_\phi(h) \right]\\
	&\approx \frac{1}{\sigma^2} \left[ \frac{\partial}{\partial \phi} g_\phi(h) \right]  \delta_1 \left.\left[ \frac{\partial}{\partial x} u_\theta(x) \right] \right|_{x=g_\phi(h)}\\
	&=\frac{1}{\sigma^2} \delta_1 \left[ \frac{\partial}{\partial \phi} u_\theta( g_\phi(h) ) \right]
	\end{split}
	\end{displaymath}
	where $\frac{\partial}{\partial \phi} g_\phi(h)$ is a matrix of size $dim(\phi) \times dim(x)$.
	Therefore, the inclusive-NRF Algorithm \ref{alg:learning-NRF-IAG} can be viewed to perform  the following steps:
	\begin{enumerate}
		\item Draw an empirical example $\tilde{x} \sim \tilde{p}(\tilde{x})$.
		\item Draw $h \sim p(h)$, $x' = g_\phi(h)$, and generate $x$ by one-step-gradient according to Eq. (\ref{eq:SGLD-one-step}).
		\item Update $\theta$ by ascending: $\nabla_\theta u_\theta(\tilde{x}) - \nabla_\theta u_\theta(x)$.
		\item Update $\phi$ by descending: $-\frac{\partial}{\partial \phi} u_\theta( g_\phi(h) )$.
	\end{enumerate}
	
	Now suppose that we interpret the potential function $u_\theta(x)$ as the discriminator in GANs (or the critic in Wasserstein GANs), which assign high scalar scores to empirical samples $\tilde{x} \sim \tilde{p}(\tilde{x})$ and low scalar scores to generated samples $x$.
	Then, the inclusive-NRF training could be viewed as playing a two-player minimax game:
	\begin{equation} \label{eq:NRF-GAN-obj}
	\min_\phi \max_{\theta} E_{\tilde{x} \sim p_0} \left[u_\theta(\tilde{x})\right]
	- E_{h \sim p(h)} \left[  u_\theta(g_{\phi}(h)) \right],
	\end{equation}
	\emph{except that} in optimizing $\theta$, the generated sample are further revised by taking one-step-gradient of $u_\theta(x)$ w.r.t. $x$ (as shown in the above Step 2.
	The `discriminator' $u_\theta$ is trained to discriminate between empirical samples and generated samples, while the generator $q_\phi$ is trained to fool the discriminator by assigning higher scores to generated samples.
	From the above analysis, we find some interesting connections between inclusive-NRFs and existing studies in GANs.
	\begin{itemize}
		\item  The optimization shown in Eq. (\ref{eq:NRF-GAN-obj}) is in fact the same as that in Wasserstein GANs (Theorem 3 in \cite{arjovsky2017wasserstein}), \emph{except that} in Wasserstein GANs, the critic $u_\theta(x)$ is constrained to be 1-Lipschitz continuous.
		%So hopefully we can improve the inclusive-NRF training by constraining the discriminator $u_\theta(x)$ to be 1-Lipschitz continuous, e.g. by utilizing the recently developed technique of spectral normalization of weight matrices in the discriminator as in \cite{Miyato2018SpectralNF}.
		\item To optimize $\theta$, the generated sample is obtained by taking one-step-gradient of $u_\theta(x)$ w.r.t. $x$.
		The tiny perturbation guided by the gradient to increase the score for the generated sample in fact creates an adversarial example.
		A similar idea is presented in \cite{liu2018adversarial} that when feeding real samples to the discriminator, 5 steps of PGD  (Projected Gradient Descent) attack is taken to decrease the score to create adversarial samples. 
		It is shown in \cite{liu2018adversarial} that training the discriminator with adversarial examples significantly improves the GAN traning.
%		Hopefully in training the discriminator in inclusive-NRFs, the adversarial attack could be increasing scores for generated samples, or decreasing scores for real samples, or a mixed one.
		\item The above analysis assume the use of one-step SGLD. It can be seen that running finite steps of SGLD in sample revision in fact create adversarial samples to fool the discriminator.
	\end{itemize} 
	
	\section{Details of experiments}
	\label{sec:exp-details}
	
	\subsection{Regularization}
	\label{sec:reg-loss}
	
	Apart from the basic loss as shown in Eq.(\ref{eq:jrf_unsup_gradient}), the following potential control loss is empirically found to stabilize the training of inclusive-NRFs.
	For random fields, the data log-likelihood $log p_\theta(\tilde x)$ is determined relatively by the potential value $u_\theta(\tilde x)$. 
	To avoid the potential values not to increase unreasonably, we could control the squared potential values, by minimizing:
	\begin{displaymath}
	\label{eq:Ls}
	L_p(\theta) =  {{E_{\tilde p(\tilde x)}}\left[ u_\theta(\tilde x) \right]^2}
	\end{displaymath}
	In this manner, the potential values would be attracted to zeros.
	In practice, we add stochastic gradients of $\alpha_p {L_p(\theta)}$ over minibatches to the original stochastic gradients of $\theta$ in Algorithm \ref{alg:learning-NRF-IAG}, with hyper-parameter $\alpha_p$.
	
	\subsection{Image generation on CIFAR-10}
	\label{sec:detail-cifar-10}
	
	\textbf{~~~Network architectures.} 
For a NRF and a GAN, each consists of two neural networks.
We use the same architecture for the potential $u_\theta$ in NRFs and the discriminator in GANs, and also the same architecture for the generators in both NRFs and GANs.
	Specifically, we use the network architectures in Table 
	3 of \cite{Miyato2018SpectralNF} for CNNs and Table 4 of \cite{Miyato2018SpectralNF} for ResNets.
		
	%For supervised learning, we use the semi-supervised inclusive-NRF Algorithm \ref{alg:semi-learning-NRF-IAG} over all labeled images. 
	%The difference in network architectures used for semi-supervised and unsupervised learning of inclusive-NRFs is that for SSL, the output layer of the potential network contains $K=10$ scalar units, while a single scalar output unit is used for unsupervised learning.
	
	\textbf{Hyperparameters.} We use Adam optimizer with the hyperparameter ($\beta_1=0,\beta_2=0.9$ and $\alpha=0.0003$ for random fields, $\alpha=0.0001$ for generators).
	For sample revision in inclusive-NRFs, we empirically choose SGLD with $L=1$ ($\delta_l=0.0003$) and use spectral normalization.
	The weight for the potential control loss is $\alpha_p=0.1$.

	\textbf{Evaluation.} Figure \ref{fig:generate} show the generated samples from inclusive-NRFs.
	We calculate inception score (IS) and Frechet inception distance (FID) in the same way as in \cite{Miyato2018SpectralNF}.
	We trained 10 models with different random seeds, and then generate 5000 images 10 times and compute the average inception score and the standard deviation.
	We compute FID between the empirical distribution and the generated distribution empirically over 10000 (test set) and 5000 samples.
	
	\subsection{Ablation study of inclusive-NRFs on CIFAR-10}
	\label{sec:detail_ablation}
	We use the same networks as in Table 3 in \cite{Miyato2018SpectralNF} (standard CNNs), except that the spectral normalization is replaced by batch normalization, as discussed in Section \ref{sec:ablation}.
	We use Adam optimizer with the hyperparameter ($\alpha=0.0002,\beta_1=0,\beta_2=0.9$).
We use ($\delta_l=0.0001$) for SGLD, and ($\beta=0.5,\delta_l=0.0003$) for SGHMC.
	The weight for the potential control loss is $\alpha_p=0.1$.
	
	\subsection{Anomaly detection}
	\label{sec:detail_anomaly}
	For anomaly detection, we train inclusive-NRFs with Algorithm \ref{alg:learning-NRF-IAG} and the potential control loss.  
	The network architectures and hyperparameters for KDDCUP, MNIST and CIFAR-10 datasets are listed in Table \ref{tab:kddcup_ad}, \ref{tab:mnist_ad} and \ref{tab:cifar_ad}, respectively.
	For sample revision, we use SGLD ($\delta_l=0.003$) and SGHMC ($\beta=0.3, \delta=0.03$) for KDDCUP, SGLD ($\delta_l=0.001$) and SGHMC ($\beta=0.5,\delta_l=0.003$) for MNIST and SGLD ($\delta_l=0.03$) and SGHMC ($\beta=0.9, \delta=1$) for CIFAR-10.
	
	\section{Latent space interpolation} \label{sec:latent-space-interp}
	
	Figure \ref{fig:inp} shows that the auxiliary generator smoothly outputs transitional samples as the latent code $h$ moves linearly in the latent space. 
	The interpolated generation demonstrates that the model has indeed learned an abstract representation of the data.
	
	\begin{table}[h]
		\centering
		\caption{Network architectures and hyperparameters in the 2D GMM experiment for  both NRFs and GANs.
		}
		\label{tab:gmm-detail}
		\begin{tabular}{cc}
			\toprule
			\textbf{Potential/Discriminator}  &\textbf{Generator}\\ 
			\midrule
			Input 2-dim data  &Noise $h$ (2-dim)\\    
			MLP 100 units, Leaky ReLU, Weight norm&MLP 50 units, ReLU, Batch norm\\
			MLP 100 units, Leaky ReLU, Weight norm&MLP 50 units, ReLU, Batch norm\\
			MLP 1 unit, Linear, Weight norm&MLP 2 units, Linear\\
			\midrule
			Batch size&100\\
			Number of iterations& 160,000\\
			Leaky ReLU slope&0.2\\
			Learning rate&0.001\\
			Optimizer&Adam ($\beta_1=0.5, \beta_2=0.9$)\\
			\bottomrule
		\end{tabular}
	\end{table}

		\begin{table}[htb]
		\centering
		\caption{Network architectures and hyperparameters for anomaly detection on KDDCUP dataset.
		}
		\label{tab:kddcup_ad}
		\begin{tabular}{cc}
			\toprule
			\textbf{Potential}  &\textbf{Generator}\\ 
			\midrule
			Input 120-dim data  &Noise $h$ (5-dim)\\    
			MLP 60 units, Tanh, Weight norm&MLP 10 units, Tanh, Batch Norm\\
			MLP 30 units, Tanh, Weight norm&MLP 30 units, Tanh, Batch Norm\\
			MLP 10 units, Tanh, Weight norm&MLP 60 units, Tanh, Batch Norm\\
			MLP 1 unit, Linear, Weight norm&MLP 120 unit, Linear, Weight norm\\
			\midrule
			Batch size&1024\\
			Number of epochs& 30\\
			Learning rate&0.0001 fro RF, 0.0003 for G\\
			Optimizer&Adam ($\beta_1=0.5, \beta_2=0.999$)\\
			Sample revision steps &$L=10$\\
			$\alpha_p$ & 0.1\\
			\bottomrule
		\end{tabular}
	\end{table}

	\begin{table*}[htb]
	\centering 
	\caption{Network architectures and hyperparameters for anomaly detection on MNIST dataset.}
	\label{tab:mnist_ad}
	\begin{tabular}{cc} 
		\toprule
		\textbf{Potential}  &\textbf{Generator}\\  
		\midrule
		Input $28\times28$ Gray Image &Noise $h$ (100-dim)\\    
		MLP 1000 units, Leaky ReLU, Weight norm &MLP 500 units, Sotfplus, Batch norm\\
		MLP 500 units, Leaky ReLU, Weight norm &MLP 500 units, Sotfplus, Batch norm\\
		MLP 250 units, Leaky ReLU, Weight norm &MLP 784 units, Sigmoid\\
		MLP 250 units, Leaky ReLU, Weight norm &\\
		MLP 250 units, Leaky ReLU, Weight norm &\\
		MLP 1 units, Linear, Weight norm &\\
		\midrule
		Batch size&100\\
		Number of epochs& 50\\
		Leaky ReLU slope&0.2\\
		Learning rate&0.003 for RF, 0.001 for G\\
		Optimizer&Adam ($\beta_1=0.0, \beta_2=0.9$)\\
		Sample revision steps &$L=20$\\
		$\alpha_p$ & 1\\
		\bottomrule
	\end{tabular}
\end{table*}

\begin{table*}[htb]
	\centering  
	\caption{Network architectures and hyperparameters for anomaly detection on CIFAR-10 dataset.}
	\label{tab:cifar_ad}
	\begin{tabular}{cc} 
		\toprule
		\textbf{Potential}  &\textbf{Generator}\\ 
		\midrule
		Input $32\times32$ Color Image &Noise $h$ (100-dim)\\    
		$3\times3$ conv. 96, Leaky ReLU, Weight norm &MLP 8192 units, ReLU, batch norm\\
		$3\times3$ conv. 96, Leaky ReLU, Weight norm &Reshape $512\times4\times4$\\
		$3\times3$ conv. 96, stride=2, Leaky ReLU, Weight norm &$5\times5$ deconv. 256, ReLU, Stride=2\\
		$3\times3$ conv. 192, Leaky ReLU, Weight norm &$5\times5$ deconv. 128 ReLU, stride=2\\
		$3\times3$ conv. 192, Leaky ReLU, Weight norm &$5\times5$ deconv. 3, Tanh, Stride=2\\
		$3\times3$ conv. 192, stride=2, Leaky ReLU, Weight norm &\\
		$3\times3$ conv. 192, Leaky ReLU, Weight norm &\\
		$1\times1$ conv. 192, Leaky ReLU, Weight norm &\\
		$1\times1$ conv. 192, Leaky ReLU, Weight norm &\\
		MLP 1 units, Linear, Weight norm &\\
		\midrule
		Batch size&64\\
		Number of epochs& 100\\
		Leaky ReLU slope&0.2\\
		Learning rate&0.001\\
		Optimizer&Adam ($\beta_1=0.0, \beta_2=0.9$)\\
		Sample revision steps &$L=10$\\
		$\alpha_p$ & 0.1\\	
		\bottomrule
	\end{tabular}
\end{table*}
		\begin{figure*}
		%\vskip -0.2in
		
		\centering  
		\includegraphics[width=0.8\textwidth]{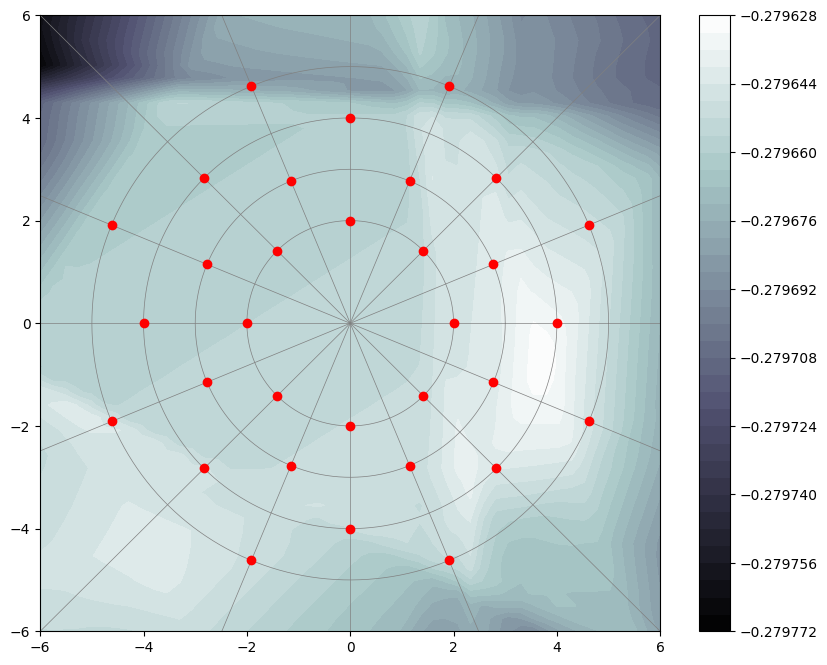}    
		
		%		\vskip -0in
		\caption{The learned NRF potential $u_\theta(x)$ from applying the contrastive divergence (CD) algorithm over the GMM synthetic data.
We encourage readers to compare with (g) and (h) in Figure \ref{fig:toy}, which are learned potentials from exclusive and inclusive NRFs respectively.}
		\label{fig:CD-potential-map}
		%		\vskip -0.0in
	\end{figure*}

	\begin{figure*}
		%\vskip -0.2in

		\centering  
		\includegraphics[width=0.8\textwidth]{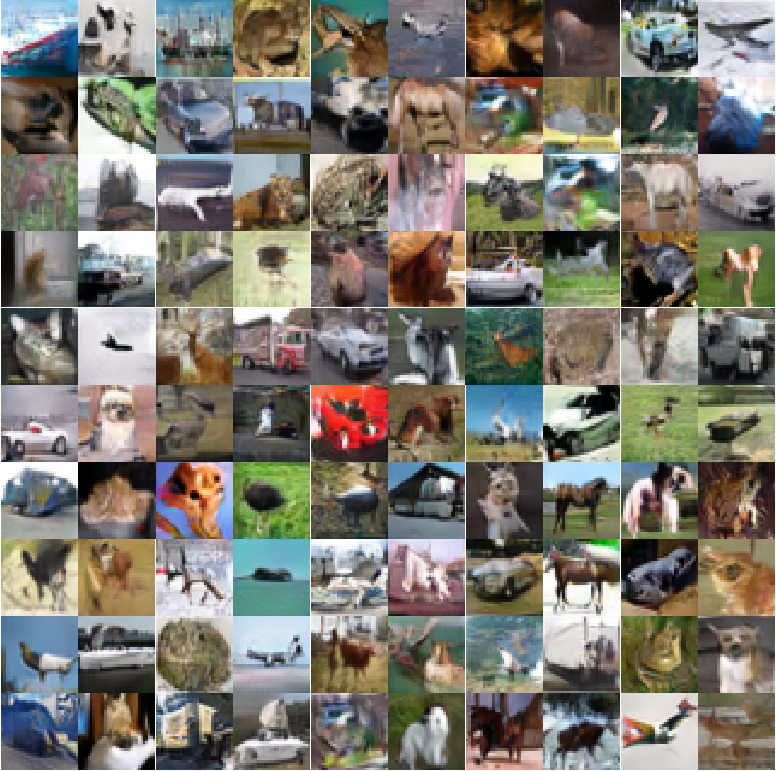}    
  
%		\vskip -0in
		\caption{Generated samples from inclusive-NRFs on CIFAR-10.}
		\label{fig:generate}
%		\vskip -0.0in
	\end{figure*}

\begin{figure*}[htb]
	\centering  
	\includegraphics[width=0.8\textwidth]{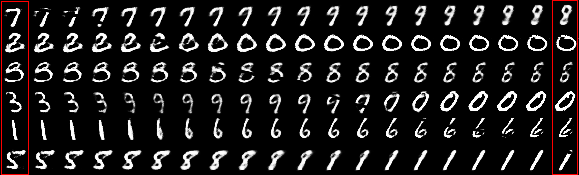}  
	\caption{Latent space interpolation with inclusive-NRFs on MNIST.
		The leftmost and rightmost columns are from stochastic generations $x_1$ with latent code $h_1$ and $x_2$ with $h_2$, respectively.
		The columns in between correspond to the generations from the latent codes interpolated linearly from $h_1$ to $h_2$.
	} 
	\label{fig:inp}
\end{figure*}  
	
	\medskip
	
	\small
	
\end{document}